\documentclass[acmsmall,nonacm]{acmart}
\AtBeginDocument{%
  }

\setcopyright{acmlicensed}
\acmDOI{} 

\acmJournal{TELO}

\graphicspath{{./figures/}}
\usepackage{subcaption}
\usepackage[ruled,boxed,linesnumbered]{algorithm2e}
\usepackage{array}
\usepackage{multirow}
\begin{document}

\title{Decomposability-Guided Cooperative Coevolution for Large-Scale Itinerary Planning}


\author{Ziyu Zhang}
\affiliation{%
	\institution{Nanjing University of Information Science and Technology}
	\department{School of Artificial Intelligence}
	\city{Nanjing}
	\state{Jiangsu}
	\country{China}
}
\email{ziyuzhang@mail.ustc.edu.cn}
 \orcid{0009-0000-9712-3821}

\author{Peilan Xu}
\correspondingauthor
\affiliation{%
	\institution{Nanjing University of Information Science and Technology}
	\department{School of Artificial Intelligence}
	\city{Nanjing}
	\state{Jiangsu}
	\country{China}
}
\email{xpl@nuist.edu.cn}
\orcid{0000-0002-6627-2514}

\author{Yuetong Sun}
\affiliation{%
	\institution{Nanjing University of Information Science and Technology}
	\department{School of Artificial Intelligence}
	\city{Nanjing}
	\state{Jiangsu}
	\country{China}
}
\email{19031775957@163.com}
\orcid{0009-0003-6068-7051}

\author{Yuhui Shi}
\affiliation{%
	\institution{Southern University of Science and Technology}
	\department{Department of Computer Science and Engineering}
	\city{Shenzhen}
	\state{Guangdong}
	\country{China}
}
\email{shiyh@sustech.edu.cn}
\orcid{0000-0002-8840-723X}

\author{Wenjian Luo}
\affiliation{%
	\institution{Harbin Institute of Technology}
	\department{Guangdong Provincial Key Laboratory of Novel Security
		Intelligence Technologies, Institute of Cyberspace Security,
		School of Computer Science and Technology}
	\city{Shenzhen}
	\state{Guangdong}
	\country{China}
}
\email{luowenjian@hit.edu.cn}
\orcid{0000-0002-8357-1655}

\renewcommand{\shortauthors}{Zhang et al.}

\begin{abstract}
	Large-scale itinerary planning is a variant of the traveling salesman problem, aiming to determine an optimal path that maximizes the collected point-of-interest (POI) scores while minimizing travel time and cost, subject to travel duration constraints. This paper analyzes inter-component dependence in large-scale itinerary planning and establishes sufficient conditions for weak decomposability, under which an optimal path can visit each cluster as a contiguous subpath. Guided by this weak decomposability, we propose a novel multi-objective cooperative coevolutionary algorithm for large-scale itinerary planning, addressing the challenges of component imbalance and interactions. Specifically, we design a dynamic decomposition strategy based on the normalized fitness within each component, define optimization potential considering component scale and contribution, and develop a computational resource allocation strategy. Finally, we evaluate the proposed algorithm on 18 real-world test cases, eight large-scale semi-synthetic test cases, and eight adapted public clustered orienteering problem benchmark cases. Comparative results show that the proposed algorithm achieves the highest mean hypervolume across all 34 test cases and significantly outperforms the compared algorithms. Ablation and optimizer-decoupling analyses further validate the effectiveness of its core mechanisms and the general applicability of the proposed framework, with particularly pronounced performance gains on large-scale benchmark instances.
\end{abstract}




\received{8 October 2025}
\received[revised]{4 April 2026}
\received[revised]{27 September 2026}
\received[accepted]{7 October 2026}

\maketitle
\section{Introduction}
\label{sec: introduction}
Itinerary planning is a class of the orienteering problem, where a traveler aims to determine an optimal route within a city under given duration constraints, selecting a subset of points of interest (POIs) to maximize the total collected score \cite{gunawan2016orienteering}. It can be seen as a variant of the traveling salesman problem (TSP) and a combination of the knapsack problem and TSP \cite{vansteenwegen2011orienteering}. As a real-world application, itinerary planning not only seeks to maximize the overall travel experience, i.e., the total collected score, but also considers objectives such as minimizing travel time and cost. While some travelers manually plan their trips, this process is time-consuming and unlikely to yield Pareto-optimal solutions \cite{rani2018development}. Consequently, automated itinerary planning has gained increasing attention, offering travelers more efficient and personalized solutions \cite{chen2022automatic}. Formally studied as the tourist trip design problem (TTDP) \cite{alatiyyah2025novel}, these planning systems are increasingly incorporating complex constraints like strict time windows to reflect real-world scenarios \cite{jewpanya2025optimized}.

Automated itinerary planning algorithms can be broadly categorized into exact and approximate methods. Exact algorithms, such as mixed-integer linear programming \cite{yu2015anytime} and general algebraic modeling systems \cite{bagloee2017multi}, guarantee optimal solutions but suffer from high computational complexity. Given that itinerary planning is NP-hard \cite{li2024research,ghobadi2023integrated,gonzalez2023recommendation}, approximate methods are often preferred for efficiency. These include k-nearest neighbor algorithms \cite{castillo2008samap} and greedy algorithms \cite{wang2024research}. Among approximate approaches, metaheuristic algorithms such as particle swarm optimization (PSO) \cite{yan2023multi} and ant colony optimization (ACO) \cite{chen2023application} have gained popularity. Their population-based search mechanisms not only enhance itinerary planning performance but also facilitate the discovery of Pareto-optimal solutions, addressing the problem’s multi-objective nature \cite{ruiz2022systematic}. Recently, applying large language models (LLMs) to itinerary planning has gained traction. However, since pure LLMs often struggle with complex spatial-temporal constraints~\cite{deng2025retail,xie2024travelplanner}, recent work has combined LLM-based preference extraction with numerical solvers to improve constraint satisfaction and personalization in generated itineraries~\cite{shao2025personal}.

As the problem scale increases, particularly with the growth in the number of points of interest (POIs), traditional algorithms frequently suffer from the ``curse of dimensionality'' \cite{xu2021constraint, omidvar2015designing}. The complexity of multi-city itinerary planning, compared to single-city scenarios, stems not only from the exponential expansion of decision variables but also from the intricate structural interdependencies among multiple interconnected cities. Due to the fact that intercity travel costs are significantly higher than intracity ones, this cost disparity induces a distinct clustering characteristic within the solution space: decision variables within a city are tightly coupled, exhibiting full non-separability, whereas intercity interactions remain sparse and are restricted to boundary nodes, manifesting a prominent $k$-non-separable structure \cite{ma2018survey}. This structural paradigm of ``locally strong coupling and globally weak coupling'' renders the multi-city itinerary planning problem theoretically well-suited for cooperative coevolution methods. Although \cite{zhang2024cooperative} proposed a static decomposition scheme based on these properties, the approach overlooks the critical interaction mechanisms between components. The fundamental issue is that multi-city itinerary planning is governed by a fixed total duration, which essentially constitutes a zero-sum resource competition. An increase in the travel days allocated to one city inevitably constrains the temporal resources available to others. This dynamic coupling triggered by resource constraints cannot be effectively managed by static decomposition strategies; it not only challenges the validity of the decomposition but also imposes more rigorous demands on the algorithm's global optimization capability.  

Therefore, employing the CC framework to solve large-scale itinerary planning problems requires establishing the problem's decomposability from a graph-based model and then analyzing the imbalance and interactions among components. First, given a problem decomposition, the number and quality of POIs in each component are inherently imbalanced, reflecting disparities in tourism resources across different cities. Second, under constraints on total travel duration, all components compete for available travel days, leading to inter-component interactions. Finally, the routes formed by the endpoint of one component and the starting point of another introduce dependencies that interfere with the internal optimization of each component. To address these challenges, we propose decomposability-guided cooperative coevolution for large-scale itinerary planning. Our main contributions are as follows:

\begin{figure*}[ht!]
	\centering
	\small
	\subfloat[Large-Scale Itinerary Planning]{\includegraphics[width=0.3\textwidth]{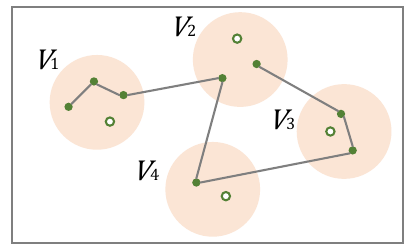}\label{fig:LSIP}}
	\subfloat[Generalized TSP] {\includegraphics[width=0.3\textwidth]{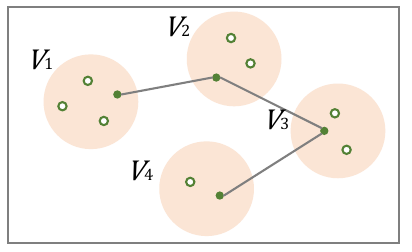}\label{fig:GTSP}}
	\subfloat[Clustered TSP] {\includegraphics[width=0.3\textwidth]{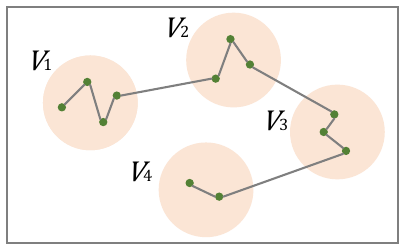}\label{fig:CTSP}}
	\caption{The comparative schematic of the generalized TSP, clustered TSP, and large-scale itinerary planning. Orange backgrounds indicate different clusters, solid circles represent visited vertices, and hollow circles denote unvisited vertices.}
	\label{fig:compar}
	\Description{Three schematic routes contrast large-scale itinerary planning, the generalized traveling salesman problem, and the clustered traveling salesman problem. Orange regions mark clusters; filled and hollow circles distinguish visited and unvisited vertices.}
\end{figure*}

\begin{itemize}
	
	\item We characterize how boundary connections couple component subpaths, and distinguish strict decomposability from weak decomposability. We establish sufficient graph and feasibility conditions under which an optimal path visits each cluster as a contiguous subpath, providing a theoretical basis for component-level cooperative search.
	
	\item Building on this weakly decomposable structure, we design a novel CC framework for large-scale itinerary planning. First, we propose a dynamic decomposition method based on the hypervolume (HV) indicator of the normalized fitness within each component. Then, we quantify the optimization potential of each component based on its scale and contribution and develop a computational resource allocation strategy. Finally, we optimize each component using NSGA-II, where inter-component cooperation is achieved through solutions with the highest hypervolume contribution.
\end{itemize}

We conduct extensive experiments on three complementary test sets. The real-world test set comprises 420 POIs from seven cities in China, 420 POIs from seven cities in France, and 240 POIs from four cities in Germany, yielding 18 test cases. We further evaluate DGCC on eight large-scale semi-synthetic test cases and eight test cases adapted from public clustered orienteering problem benchmarks. Compared with state-of-the-art multi-objective itinerary planning algorithms, DGCC achieves the highest mean hypervolume across all 34 test cases and significantly outperforms the compared algorithms. Additional ablation and optimizer-decoupling analyses validate the effectiveness of the core mechanisms and the general applicability of the proposed framework, with particularly pronounced performance gains on large-scale benchmark instances.

The rest of this paper is organized as follows. Section \ref{sec:related work} introduces the large-scale itinerary planning problem and reviews multi-objective itinerary planning algorithms as well as cooperative coevolutionary algorithms. Section \ref{sec: algorithm} derives the definition of weak decomposability along with its corresponding graph structures and provides a detailed description of the proposed DGCC. Section \ref{sec:experiments} presents the experimental setup, including data collection and parameter settings, and reports the experimental results. Section \ref{sec:Discussion} discusses the applicability and limitations of the proposed framework, as well as its potential extensions to more general problem structures. Finally, Section \ref{sec: Conclusion}  concludes the paper.

\section{Preliminaries}
\label{sec:related work}
In this section, we first present the problem formulation of the large-scale itinerary planning problem and discuss heuristic and evolutionary computation approaches designed for multi-objective itinerary planning. We then review cooperative coevolution for large-scale optimization.

\subsection{Large-Scale Itinerary Planning}
\label{sec:PSMF}

In this paper, we use the term ``large-scale itinerary planning'' to refer to multi-city itinerary planning problems in which inter-city time allocation and intra-city routing decisions must be optimized jointly under a global time budget constraint. This setting fundamentally differs from traditional single-city itinerary planning in terms of decision structure and problem coupling. A formal problem description is provided below.

Let \( G = (V, E) \) be an undirected, complete, weighted graph, where \( V = \{v_1, v_2, \dots\} \) represents the set of vertices and \( E = \{e_1, e_2, \dots\} \) represents the set of edges. Each edge connects two distinct vertices \( v_i \) and \( v_j \in V \), with \( i \neq j \), and is associated with a non-negative weight vector \( \mathbf{w}(v_i, v_j) \). Additionally, each vertex \( v_i \in V \) is also associated with a non-negative weight vector \( \mathbf{w}(v_i) \).

In the large-scale itinerary planning problem, the set of vertices \( V \) is partitioned into \( m \) disjoint clusters \( V_1, V_2, \dots, V_m \), satisfying \( V_i \cap V_j = \emptyset \) for all \( i \neq j \), with \( V = \bigcup_{i=1}^m V_i \). Each cluster must be visited at least once, and each vertex \( v \) can be visited at most once. The objective of the large-scale itinerary planning problem is to select a subset of \( V \) and determine the Pareto optimal path under given constraints, such as duration limitations, to optimize the objective function vector.

The large-scale itinerary planning problem is a variant of the TSP, distinguished by the partitioning of the vertex set and constrained vertex visitation rules. Similar TSP variants that incorporate vertex set partitioning include the generalized TSP \cite{srivastava1969generalized} and the clustered TSP \cite{chisman1975clustered}. Figure \ref{fig:compar} illustrates the differences among these three TSP variants. In the generalized TSP, the objective is to find the shortest Hamiltonian cycle while ensuring that exactly one vertex from each cluster is visited, as shown in Figure \ref{fig:GTSP}. The clustered TSP follows a similar structure but requires visiting all vertices within each cluster, as shown in Figure \ref{fig:CTSP}. Evidently, both variants restrict each cluster to be visited only once. In contrast, the large-scale itinerary planning problem requires visiting more than one vertex per cluster and allows multiple visits to the same cluster, as shown in Figure \ref{fig:LSIP}. Thus, it can be regarded as a more general form of both the generalized TSP and the clustered TSP. The fundamental challenge of the aforementioned LSIP problem lies in the severe combinatorial explosion triggered by its multi-city nature. Unlike traditional single-city routing, LSIP requires the simultaneous determination of optimal node selections and path sequences within a massive Cartesian product of multiple independent urban solution spaces. For conventional metaheuristic or holistic algorithms that attempt to optimize all decision variables as a unified entity, this exponential surge in dimensionality often leads to low search efficiency and premature convergence to local optima. Moreover, the structural separability makes the divide-and-conquer philosophy of cooperative coevolution theoretically well-suited for this problem. By decomposing the global task into manageable, city-based subcomponents, the CC framework inherently constrains the solution structure, effectively filtering out inefficient topological patterns from the search space and mitigating the curse of dimensionality. However, directly applying conventional CC to LSIP remains non-trivial due to the dynamic coupling among subproblems. Because all cities are tightly interconnected through a zero-sum competition for a fixed total travel duration, treating components as completely independent static entities inevitably compromises global optimality. This dynamic dependency dictates that a successful CC approach must incorporate adaptive problem decomposition and dynamic resource allocation mechanisms.

A considerable body of research has focused on solving multi-objective itinerary planning problems, employing both heuristic and evolutionary computation approaches. Several studies have utilized heuristic algorithms. For instance, \cite{castillo2008samap} applied the k-nearest neighbor algorithm in the SAMAD tourism planning tool, while \cite{brilhante2015planning} used a greedy algorithm as the foundation of TRIPBUILDER, an unsupervised framework for personalized city sightseeing tours based on public data from platforms such as Wikipedia and Flickr. In contrast, many studies have leveraged evolutionary computation techniques.  \cite{jean2005challenges} formulated multi-objective optimization problems for real-world itinerary planning and solved them using multi-objective simulated annealing. This study focused on optimizing sightseeing itineraries that integrate hotel stays and transportation options such as buses and cars, providing decision support for travelers. \cite{yan2023multi} proposed an improved particle swarm optimization algorithm to enhance flexibility and personalization, addressing the limitations of single-itinerary approaches. \cite{kolaee2023local} introduced a multi-objective optimization model for medical tourism, aiming to minimize costs while maximizing trip attractiveness through patient-to-hospital assignment and city visit scheduling. Their local search-based NSGA-II achieved a balanced trade-off between cost and attractiveness within an hour. \cite{huang2024uncertain} developed a multi-objective programming model under uncertainty to optimize travel time, cost, and tourist satisfaction. By transforming the model using inverse uncertainty distribution and solving it with an ant colony optimization, they demonstrated its effectiveness through a numerical example.  Alongside these evolutionary methods, Deep Reinforcement Learning (DRL) has recently emerged as a promising approach to automate routing decisions \cite{chen2022automatic}. However, since pure neural combinatorial optimization models often suffer from severe performance degradation and poor generalization when scaled to massive, unseen graph structures \cite{luo2023neural}, hybridizing DRL with metaheuristics has become a popular strategy to balance rapid inference and robust global search \cite{pitakaso2025hybrid}.

\subsection{Cooperative Coevolution}

Cooperative coevolution (CC) \cite{potter1994cooperative} is an effective approach for solving large-scale optimization problems. It follows a divide-and-conquer strategy, decomposing a large-scale problem into lower-dimensional components that are evolved cooperatively. CC has also been increasingly applied to multi-objective optimization \cite{ficici2008multiobjective,antonio2015non}. Iorio and Li \cite{iorio2004cooperative} proposed a cooperative coevolutionary multi-objective algorithm that employs non-dominated sorting to reward successful collaborations among sub-populations. \cite{gong2016multiobjective} developed a multi-objective cooperative coevolutionary algorithm for hyperspectral sparse unmixing, addressing the challenges of the NP-hard \( l_0 \)-norm problem by simultaneously optimizing reconstruction, sparsity, and regularization terms.  \cite{antonio2015non} introduced a collaboration formation mechanism for cooperative coevolutionary multi-objective evolutionary algorithms, leveraging the hypervolume indicator to select individuals from different species to form solutions.

Problem decomposition is a critical step in CC \cite{omidvar2021review,ma2018survey,antonio2017coevolutionary}. However, in most real-world problems, variables exhibit interactions, and the optimal decomposition strategy is generally unknown \cite{xu2023large}, prompting the development of more adaptive decomposition methods. For instance,  \cite{li2021mlfs} proposed a clustering-based decomposition strategy to reduce the computational cost of identifying an optimal decomposition.  \cite{guo2023cooperative} developed a fuzzy decomposition method that incorporates feedback during optimization, improving adaptability. Meanwhile, some studies have focused on the allocation of computational resources to subcomponents. Traditional CC frameworks distribute computational resources equally among all components. However, \cite{omidvar2011smart} identified imbalances in component contributions and introduced the contribution-based cooperative coevolution framework to address this issue.  \cite{luo2019many} proposed difficulty-based cooperative coevolution to handle subproblems of varying complexity.  \cite{xu2023difficulty} extended this approach by integrating both contribution and difficulty metrics, introducing difficulty-contribution-based cooperative coevolution to mitigate inaccuracies in contribution assessment caused by component difficulty. Additionally, \cite{liang2018novel} developed a resource reallocation mechanism that detects stagnation in the optimization process and redistributes computational resources accordingly.

\section{The Proposed Algorithm}
\label{sec: algorithm}

The efficiency of large-scale optimization algorithms largely depends on their ability to effectively exploit problem structures \cite{omidvar2021review}. In large-scale itinerary planning, the set of POIs, denoted as \( V \), is partitioned into clusters \( \{V_1, V_2, \dots, V_m\} \). Thus, a fundamental consideration is how to establish the equivalence between itinerary planning within each cluster and planning over the entire set \( V \). Furthermore, given a specific decomposition scheme, it is essential to address the interactions between components and identify their inherent properties. Specifically, independently planned itineraries within clusters still exhibit spatiotemporal interactions. Temporally, when the total travel duration is fixed, the assignment of travel days across cities must be considered. Spatially, the independence of itineraries across cities depends on intercity travel distances and urban transportation networks between POIs. To tackle these challenges, this section first analyzes the problem’s decomposability and introduces the definition of weak decomposability. Then, we present the core steps of CC for large-scale multi-objective itinerary planning, including problem decomposition, computational resource allocation, and component coevolution, along with a complete framework.

\subsection{Decomposability Analysis}
\label{decomposability_analysis}
To mitigate the curse of dimensionality in large-scale itinerary planning, treating each cluster as a component and optimizing them independently is a promising approach. However, this strategy relies on two key assumptions: (1) each cluster is visited exactly once, and (2) itineraries within each cluster are independent of each other. Nevertheless, the optimal solution obtained under these assumptions does not always align with the optimal solution of the original problem. Therefore, it is essential to analyze the decomposability of large-scale itinerary planning.

Similar to general large-scale optimization problems, a large-scale itinerary planning problem is defined as decomposable if it can be partitioned into several components that can be solved separately. This is formally stated in definition as follows.

\textbf{Definition 1} (Strict Decomposability)\textbf{.}
For the scalar path objective considered in this analysis, a large-scale itinerary planning problem is said to exhibit \emph{strict decomposability} if and only if each cluster \(V_i\) admits an optimal subpath \(\mathbf{x}_i^\star\) that can be determined independently of the boundary connections, and there exists a permutation \(\pi\) for which their concatenation is a globally optimal feasible path:
\[
\mathbf{x}^\star=\Pi_\pi(\mathbf{x}_1^\star,\ldots,\mathbf{x}_m^\star),
\qquad
\mathbf{x}^\star\in\operatorname*{arg\,min}_{\mathbf{x}\in\mathcal{X}}F(\mathbf{x}).
\]
Here, \(\mathcal{X}\) is the set of feasible global paths, \(\{V_1,\ldots,V_m\}\) is a partition of \(V\), and \(\Pi_\pi\) concatenates the component subpaths in the order specified by \(\pi\).

Boundary connections can affect the contribution of an internally optimal subpath to the complete itinerary. The following lemma characterizes this dependence for an additive path objective.

\textbf{Lemma 1} (Boundary Dependence)\textbf{.}
Let \(\mathbf{q}=(a,\ldots,b)\) be a subpath within a cluster in an undirected graph, and let \(u\) and \(v\) be its fixed predecessor and successor outside that cluster. Suppose that both orientations of \(\mathbf{q}\) yield feasible itineraries. Reversing \(\mathbf{q}\) preserves its internal objective value. If
\[
C(u,a)+C(b,v)\neq C(u,b)+C(a,v),
\]
the two orientations have different contributions to the assembled itinerary, where \(C\) denotes the boundary-edge contribution to the objective.

\begin{proof}
	The internal vertex and edge contributions are unchanged under reversal. Only the incoming and outgoing boundary edges change, so the difference between the reversed and original itineraries is
	\[
	C(u,b)+C(a,v)-C(u,a)-C(b,v).
	\]
	This difference is nonzero under the stated condition. If the component lies at an endpoint of the complete path, the nonexistent boundary-edge term is omitted.
\end{proof}

Thus, internal objective values alone need not determine the preferred subpath orientation; its boundary connections must also be considered. To retain contiguous component visits while allowing such dependence, we introduce a weaker criterion.
We consider a necessary condition for strict decomposability based on the visit count of each cluster, as stated in Lemma 2. The visit count \(\mathcal{N}(V_i)\) of cluster \(V_i\) is defined as the number of occurrences of vertex pairs \((x_j,x_{j+1})\) in the path \(\mathbf{x}\) for which \(x_{j+1}\in V_i\) and \(x_j\in V\setminus V_i\). Additionally, the visit count of the cluster containing the first visited vertex \(x_1\) is initialized to 1.

\textbf{Lemma 2.}
A necessary condition for a large-scale itinerary planning problem to be strictly decomposable is that there exists an optimal path that visits each cluster exactly once.

\begin{proof}
	Suppose that the problem is strictly decomposable. By Definition 1, for each cluster \(V_i\), there exists an independently optimal subpath \(\mathbf{x}^{(i)}\), and a global optimal path can be obtained by concatenating these subpaths in some order. In this path, all visited vertices belonging to the same cluster form a single contiguous subpath. Hence, \(\mathcal{N}(V_i)=1\) for all \(V_i\), which proves the existence of an optimal path that visits each cluster exactly once.
\end{proof}

When a large-scale itinerary planning problem satisfies the necessary condition stated in Lemma 2, partitioning the problem into components at the cluster level results in direct interactions between components at the starting and/or ending points of subpaths. This, in turn, induces indirect interactions within each component. Similar to overlapping large-scale optimization problems, such interactions can drive the components toward convergence in a coevolutionary manner. Therefore, we define a large-scale itinerary planning problem that meets this necessary condition as weakly decomposable.

\textbf{Definition 2} (Weak Decomposability)\textbf{.}
A large-scale itinerary planning problem is said to exhibit weak decomposability if and only if there exists an optimal path \(\mathbf{x}^*\) that visits each cluster exactly once, i.e., \(\mathcal{N}(V_i)=1\) for all \(V_i\) along \(\mathbf{x}^*\).

Theorem 1 gives sufficient conditions for weak decomposability in the graph model considered below.

\textbf{Theorem 1.}
Let \(G=(V,E)\) be an undirected, complete, weighted graph with a vertex partition \(V=V_1\cup\cdots\cup V_m\) corresponding to \(m\) clusters, where \(V_i\cap V_j=\emptyset\) for \(i\neq j\). For a fixed set of visited vertices, consider the objective function
\begin{equation}\label{eqt:object}
	f(\mathbf{x})=\alpha\sum_{i=1}^{|\mathbf{x}|}w(x_i)
	+\beta\sum_{i=1}^{|\mathbf{x}|-1}w(x_i,x_{i+1}),
\end{equation}
where \(\mathbf{x}=(x_1,\ldots)\) is a path in \(G\), \(\alpha,\beta\geq0\), \(w(x_i)\) is the weight of vertex \(x_i\), and \(w(x_i,x_{i+1})\) is the weight of the edge between consecutive vertices. Any additional constraints are assumed to remain satisfied under the path reorderings used below.

The problem is weakly decomposable if
\begin{equation}\label{eqt:weak_dcomp}
	\max\{w_{\max}(V_i),w_{\max}(V_j)\}
	\leq w_{\min}(V_i,V_j),
	\qquad \forall i\neq j,
\end{equation}
and, when \(m\geq3\), if the following additional condition holds:
\begin{equation}\label{eqt:transit_margin}
	w_{\max}(V_i)\leq\sigma(V_i),
	\qquad \forall i\in\{1,\ldots,m\}.
\end{equation}
Here,
\[
w_{\max}(V_i)=\max_{\substack{u,v\in V_i\\u\neq v}}w(u,v),
\qquad
w_{\min}(V_i,V_j)=\min_{\substack{u\in V_i\\v\in V_j}}w(u,v),
\qquad
w_{\max}(V_i,V_j)=\max_{\substack{u\in V_i\\v\in V_j}}w(u,v)
\]
denote the maximum intra-cluster edge weight and the minimum and maximum inter-cluster edge weights, respectively. We set \(w_{\max}(V_i)=0\) when \(V_i\) contains only one vertex. For \(m\geq3\), the inter-cluster transit margin is
\[
\sigma(V_i)=
\min_{\substack{j,k\in\{1,\ldots,m\}\setminus\{i\}\\j\neq k}}
\left\{w_{\min}(V_i,V_j)+w_{\min}(V_i,V_k)-w_{\max}(V_j,V_k)\right\}.
\]
Condition (\ref{eqt:weak_dcomp}) bounds intra-cluster edge weights by inter-cluster edge weights. Condition (\ref{eqt:transit_margin}) additionally rules out a cost reduction from using a cluster as an intermediate connection between two other clusters.

\begin{proof}
	Consider an optimal path \(\mathbf{x}\). If every cluster is visited exactly once, the result follows. Otherwise, let two consecutive segments of a cluster \(V_i\) be separated by a nonempty subpath containing no vertex of \(V_i\). Write the relevant part of \(\mathbf{x}\) as
	\[
	(\ldots,a,p,\ldots,q,r,\ldots,s,t,\ldots),
	\]
	where \(a,r,s\in V_i\), \(p,\ldots,q\) is the intervening subpath, and \(t\) is omitted if the second segment of \(V_i\) is the last segment of the path. Since \(G\) is undirected, we may reverse the intervening subpath and the second segment of \(V_i\) to obtain
	\[
	(\ldots,a,s,\ldots,r,q,\ldots,p,t,\ldots).
	\]
	The visited vertices and their total vertex weight are unchanged. The internal edge weights of the reversed subpaths are also unchanged. Thus, only the boundary edges need to be compared.
	
	If \(t\) is absent, the edge \((a,p)\) is replaced by \((a,s)\). With \(p\in V_j\), condition (\ref{eqt:weak_dcomp}) gives \(w(a,s)\leq w_{\max}(V_i)\leq w_{\min}(V_i,V_j)\leq w(a,p)\).
	
	If \(t\) is present, let \(p\in V_j\) and \(t\in V_k\). For \(j=k\), condition (\ref{eqt:weak_dcomp}) gives
	\[
	w(a,s)+w(p,t)
	\leq w_{\max}(V_i)+w_{\max}(V_j)
	\leq 2w_{\min}(V_i,V_j)
	\leq w(a,p)+w(s,t).
	\]
	For \(j\neq k\), condition (\ref{eqt:transit_margin}) and the definition of \(\sigma(V_i)\) give
	\[
	\begin{aligned}
		w(a,s)+w(p,t)
		&\leq w_{\max}(V_i)+w_{\max}(V_j,V_k)\\
		&\leq w_{\min}(V_i,V_j)+w_{\min}(V_i,V_k)\\
		&\leq w(a,p)+w(s,t).
	\end{aligned}
	\]
	Therefore, in every case the reordering does not increase \(f\) and merges two segments of \(V_i\). It does not split any segment of another cluster. Repeating the operation finitely many times yields a feasible path \(\mathbf{x}^*\) that visits each cluster exactly once and satisfies \(f(\mathbf{x}^*)\leq f(\mathbf{x})\). Since \(\mathbf{x}\) is optimal, \(\mathbf{x}^*\) is also optimal. Hence, the problem is weakly decomposable.
\end{proof}

To illustrate these conditions, Figure \ref{fig:weak_decomposability} shows a multi-cluster graph and its corresponding cost matrix. Figure \ref{fig:weak_decomposability}(a) highlights representative inter-cluster connections; the complete set of edge weights is shown in Figure \ref{fig:weak_decomposability}(b). The low-cost diagonal blocks and higher-cost off-diagonal entries illustrate condition (\ref{eqt:weak_dcomp}), while the numerical edge weights also satisfy the additional condition (\ref{eqt:transit_margin}).

\begin{figure}[htbp]
	\centering
	\includegraphics[width=1.0\textwidth]{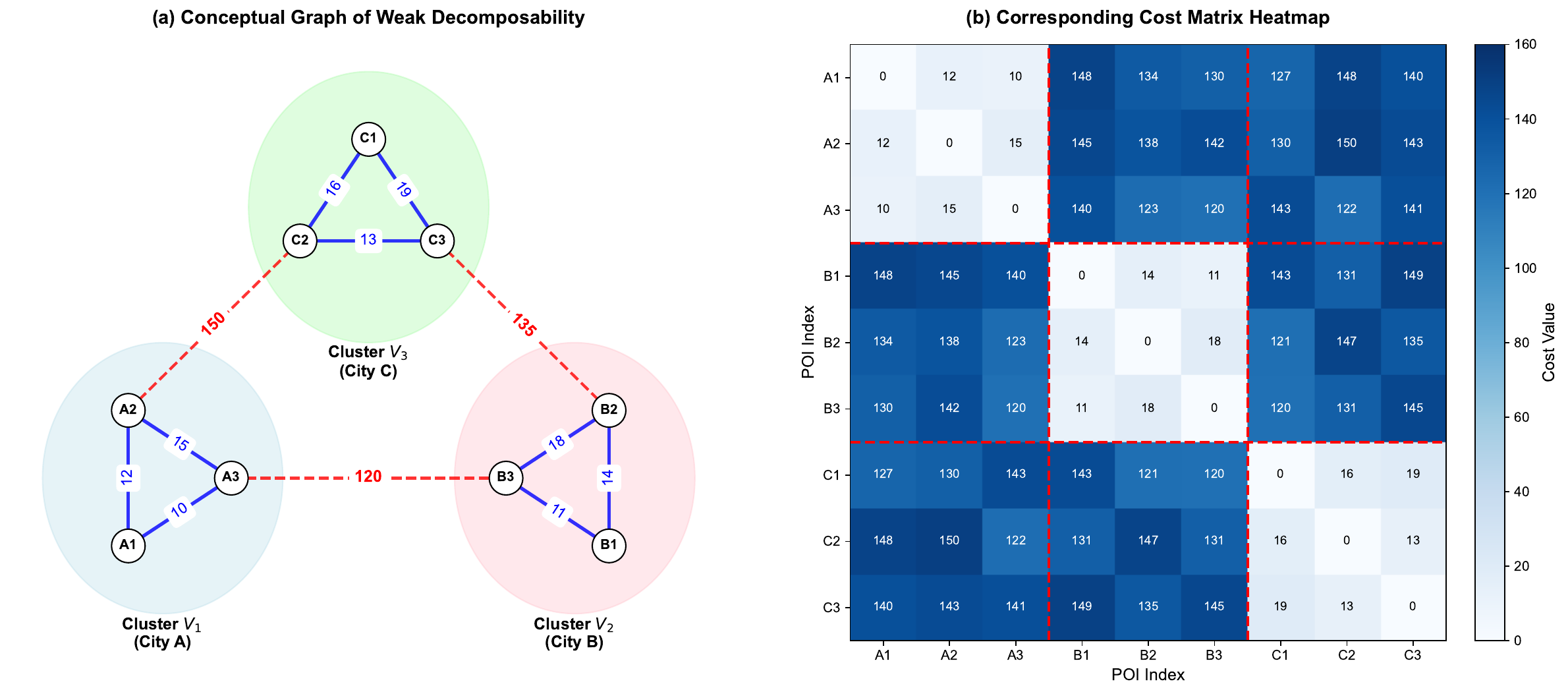}
	\caption{A visual illustration of a graph structure satisfying the weak decomposability conditions.
		(a) A conceptual multi-cluster graph highlighting intra-cluster costs and representative inter-cluster connections for visual clarity.
		(b) The corresponding complete cost matrix heatmap, displaying the characteristic block-diagonal structure.}
	\Description{Three clusters with low internal edge costs and higher inter-cluster costs, alongside their symmetric complete cost matrix satisfying the sufficient conditions in Theorem 1.}
	\label{fig:weak_decomposability} 
\end{figure}

\subsection{Encoding Scheme and Problem Decomposition}
\label{encoding_scheme}

In practical tourism scenarios, the number of POIs visited daily is typically variable due to factors such as commuting times and the duration of stay at each attraction. Consequently, adjustments to the encoding scheme are necessary. To address this issue, this paper adopts a fixed-length encoding scheme based on placeholders (referencing the structure in \cite{huang2019automatic}). The specific encoding scheme is as follows:

First, each POI is assigned a unique positive integer ID. For an itinerary with a total travel duration of $D$ days and a maximum capacity of $ M $ visitable POIs per day, the chromosome is represented as a fixed-length vector of size $D \times M$. This vector is logically partitioned into $D$ consecutive segments, with each segment corresponding to a single day's itinerary. Within each segment, positive integer genes represent the indices of the visited POIs, while a specific value of 0 serves as a placeholder, indicating no visit assignment at that position. For example, if $ D = 2 $ and $ M = 4 $, the encoding $(1, 0, 2, 0, 3, 4, 0, 0)$ represents the complete itinerary. The sequence for the first day is: Hotel $\rightarrow$ POI 1 $\rightarrow$ POI 2 $\rightarrow$ Hotel; and for the second day: Hotel $\rightarrow$ POI 3 $\rightarrow$ POI 4 $\rightarrow$ Hotel. This encoding scheme, specifically designed for itinerary planning, offers significant advantages in this scenario, particularly during optimization operations such as chromosome crossover and mutation, as well as in the decoding process~\cite{huang2019automatic}.

Guided by the decomposability analysis, we organize the search into \(m\) components using the prescribed cluster partition. However, the encoding of candidate solutions does not directly correspond to these clusters, as its length depends on the travel duration \( D \) and the maximum number of POIs visited per day \( M \). Therefore, after determining the \( m \)-component decomposition, we further define a partitioning scheme for the encoding, which consists of two phases: initial decomposition and dynamic adjustment.

\textbf{Initial Decomposition:} Before evolving the population, and without prior knowledge, we adopt an equal partitioning approach for the encoding and assign it to each component. Specifically, DGCC distributes an equal number of travel days across all cities. Given \( m \) components and an encoding length of \( D \cdot M \) for a candidate solution, the segment assigned to each component corresponds to a continuous block of length \( (D/m) \cdot M \). For simplicity, we assume that the total travel duration \( D \) is an integer multiple of the number of components \( m \).

\textbf{Dynamic Decomposition:} As the subpopulations evolve, we progressively acquire information about the problem. Based on these heuristic insights, the encoding partitioning scheme is dynamically adjusted to accommodate the inherent imbalance of the problem, particularly the uneven distribution of tourism resources across different cities.

The core idea of dynamic decomposition is to assign more basic encoding units to components whose POI sets contain a higher number of high-quality POIs. To estimate the actual contribution of each component, consider two candidate solutions \( \mathbf{x} \) and \( \mathbf{y} \), which differ only in the \( i \)-th component, i.e., \( \mathbf{x}^{(i)} \) and \( \mathbf{y}^{(i)} \), while all other components remain unchanged. The corresponding change in fitness, as defined by Equation \eqref{eqt:object}, is given by
\begin{equation}\label{eqt:diff_xy}
	\begin{aligned}
		f(\mathbf{x}) - f(\mathbf{y})
		&= f(\mathbf{x}^{(i)}) - f(\mathbf{y}^{(i)}) + \beta \left[ w(x^{(i-1)}_e, x^{(i)}_s) - w(y^{(i-1)}_e, y^{(i)}_s) \right. \\
		&\quad \left. + w(x^{(i)}_e, x^{(i+1)}_s) - w(y^{(i)}_e, y^{(i+1)}_s) \right].
	\end{aligned}
\end{equation}
where the term \( f(\mathbf{x}_i) - f(\mathbf{y}_i) \) represents the internal optimization of a component, while the remaining terms account for the external optimization effects induced by interactions with neighboring components. Thus, the internal optimization of a component reflects the characteristics of its corresponding cluster, including the number and quality of POIs. In contrast, the remaining terms capture the interactions between components, such as the edge weights between POIs in adjacent clusters.

Considering that components may have varying encoding lengths, directly adjusting the encoding length based on their internal objective values may lead to the Matthew effect or the so-called resource curse. For example, in maximizing the sightseeing score, a longer travel duration tends to yield a higher score, whereas in minimizing travel costs, a longer duration results in greater expenses. This discrepancy implies that components with different encoding lengths may be unfairly compared. To ensure a fairer evaluation, we normalize each component's fitness value by the number of basic encoding units, i.e., the assigned travel duration. Specifically, we define the normalized fitness as \( \tilde{f}(\textbf{x}^{(i)}) = f(\textbf{x}^{(i)}) / d_i \), where $d_i$ is the assigned travel days for the \( i \)-th component. Based on the normalized fitness values \( \tilde{f}(\textbf{x}^{(i)}) \) for each component, we derive the dynamic decomposition method for the problem, as outlined in Algorithm \ref{alg:Decomposition}.
\begin{algorithm}[htbp]
	\caption{Dynamic Adjustment for Components}
	\label{alg:Decomposition}
	\SetAlgoLined
	
	$HV \gets \emptyset$;
	
	\For{$i\gets 1$ \textbf{to} $m$ }{
		\ForEach{$\textbf{x}^{(i)}_j \in P_i$ }{
			Calculate $\tilde{f}(\textbf{x}_j^{(i)}) = f(\textbf{x}_j^{(i)}) / d_i$;
		}
		Calculate $HV_i$ based on $\tilde{f}(\textbf{x}_j^{(i)})$;
		
		$HV \gets HV \cup \{HV_i\}$;
	}
	
	$i_{\max} = \arg\max HV$;
	
	$i_{\min} = \arg\min \{ HV_i \in HV \mid d_i > 1 \}$;
	
	\For{$i \gets i_{\max}$ \textbf{to} $i_{\min}$}{
		\If{$i \neq i_{\min}$}{
			Add and randomly initialize an encoding unit closer to $i_{\min}$;
		}
		\If{$i \neq i_{\max}$}{
			Remove an encoding unit closer to $i_{\max}$;
		}
	}
\end{algorithm}

For each component, the normalized fitness \(\tilde{f}(\textbf{x}^{(i)})\) of its corresponding subpopulation \(P_i\) is computed for each objective, and the associated hypervolume \(HV_i\) is obtained. Next, the component with the maximum hypervolume in the set \( HV = \{HV_1, \dots, HV_m\} \) is identified as \( i_{\max} \). The component with the minimum hypervolume is selected from the subset of components whose encoding units are more than one, and is denoted as \( i_{\min} \). A basic encoding unit is then transferred from the \( i_{\min} \)-th component to the \( i_{\max} \)-th component. As illustrated in Figure \ref{fig:resource}, assuming \( i_{\max} < i_{\min} \), this transfer is performed by removing an encoding unit from the \( i_{\min} \)-th component that is closest to the \( i_{\max} \)-th component and adding it to the \( (i_{\min} - 1) \)-th component, where it is randomly initialized. Assuming \( i_{\max} < i_{\min} \), this process is repeated iteratively until all components within the interval \([i_{\max}, i_{\min}]\) have been updated.

\begin{figure}[htbp]
	\centering
	\includegraphics[width=0.8\textwidth]{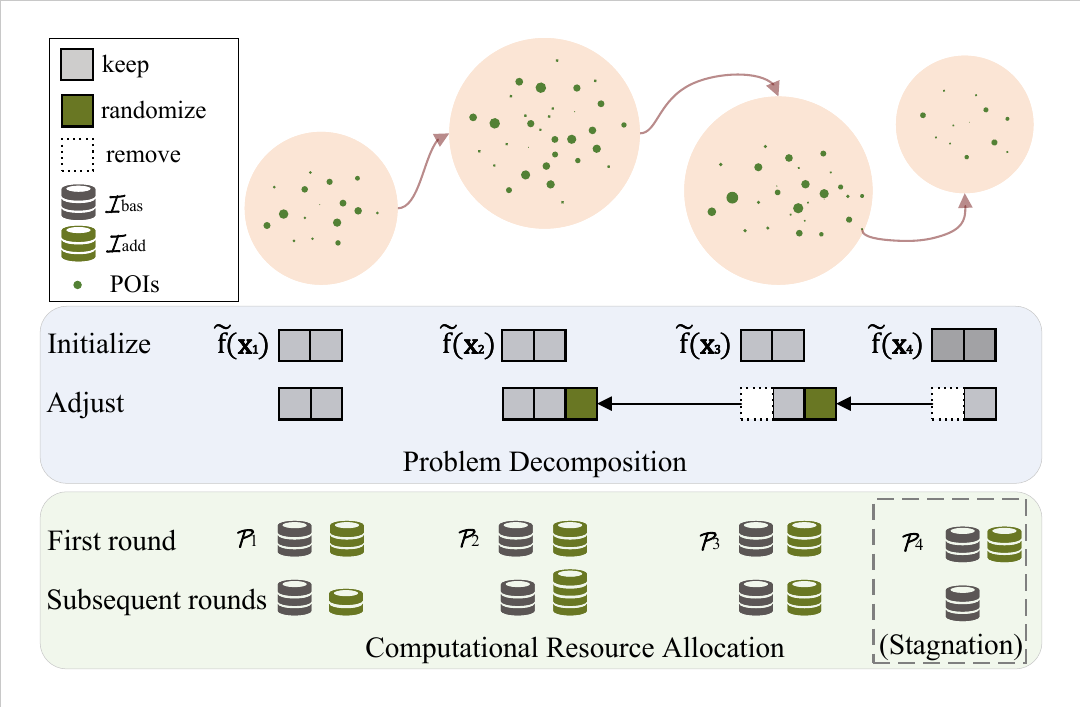}
	\caption{The schematic diagram of problem decomposition and computational resource allocation. Cities are highlighted with an orange background, while green dots within each city represent POIs, with their size indicating the quality of the corresponding POI.}
	\Description{The upper diagram shows city components and point-of-interest scores. The middle diagram shows encoding units retained, randomized, or removed during component adjustment. The lower diagram shows the initial and subsequent allocation of computational resources, including a stagnant component.}
	\label{fig:resource}
\end{figure}

\subsection{Computational Resource Allocation Strategy}\label{sec:resource}
Given a decomposition scheme for large-scale itinerary planning, the next step is to determine the available computational resources for each component to optimize. Considering the imbalance in city tourism resources, we aim to allocate computational resources preferentially towards components representing cities with abundant tourism resources to improve utilization efficiency. To achieve this, we consider both the quality and quantity of the POIs in each component. Specifically, we evaluate the improvement in hypervolume for the \(i\)-th component at the \(k\)-th iteration, denoted as \(\Delta^k_i = \mathcal{C}_{i}^k - \mathcal{C}_{i}^{k-1}\), where \(\mathcal{C}_{i}^k\) represents the hypervolume of the \(i\)-th component at the \(k\)-th iteration. Additionally, we consider the number of POIs \(\mathcal{N}_i\) in each component, which also represents the search bound of each variable for the \(i\)-th component. Therefore, we quantify the optimization potential of each component using \(\mathcal{P}_i^k\):

\begin{equation}\label{eqt:P}
	\mathcal{P}_i^k = (\Delta^k_i + \mathcal{B}^k)\cdot \mathcal{N}_i.
\end{equation}
We introduce a balancing coefficient \(\mathcal{B}\) to smooth out the differences in the improvement of hypervolume across components. This is designed to prevent extreme bias due to significant differences in the improvement of hypervolume, which could lead to poor tourism experiences for certain cities. The calculation of \(\mathcal{B}\) is as follows:

\begin{equation}\label{eqt:B}
	\mathcal{B}^k = \frac{1}{m} \sum_{j=1}^{m} \Delta^k_j + \delta,
\end{equation}
where \(m\) represents the number of components, and \(\delta\) is a small positive constant to prevent \(\mathcal{B}^k\) from becoming zero.

Due to interactions among components, we adopt a round-robin optimization approach. Computational resources are then allocated according to the optimization potential \(\mathcal{P}_i^k\) of each component. In this study, computational resources refer to the number of fitness evaluations (FEs) allocated to each component, which consists of two parts: the basic fitness evaluations \(\mathcal{I}_{\text{bas}}\) and the additional fitness evaluations \(\mathcal{I}_{\text{add}}\). The former ensures a minimum level of optimization for each component, while the latter allocates computational resources in proportion to the optimization potential of components. To differentiate between stagnant and non-stagnant components, we first define the available computational resources \(\mathcal{I}_{\text{avl}}\) for a non-stagnant component in one round, as follows:

\begin{equation}\label{eqt:avl}
	\mathcal{I}_{\text{avl}} = \mathcal{I}_{\text{bas}} + |U|\mathcal{I}_{\text{add}} \cdot \frac{\mathcal{P}_i}{\sum_{j\in U}{\mathcal{P}_j}},
\end{equation}
where \( U \) represents the set of all components that have not stagnated in evolution, and \( |U| \) denotes the cardinality of this set. In each optimization round, the total additional fitness evaluations allocated to all non-stagnant components is given by \( |U| \mathcal{I}_{\text{add}} \), which is then allocated proportionally among them. In contrast, stagnant components retain only the basic fitness evaluations \( \mathcal{I}_{\text{bas}} \). Based on empirical observations, we define a component as stagnant if it satisfies the condition \( \Delta_i^k / \mathcal{C}_{i}^{k-1} < 5 \times 10^{-5} \).

Under the proposed computational resource allocation mechanism, an increase in the number of stagnant components leads to a reduction in the additional fitness evaluations per round, thereby increasing the number of optimization rounds. This is beneficial in contexts where interactions exist between components, as stagnant components gain additional opportunities to escape local optima through updates in the subpopulation of their interacting components.

Figure \ref{fig:resource} illustrates the adopted computational resource allocation scheme. First, the optimization sequence is pre-determined based on the geographical locations of the cities. In the initial optimization round, each component receives an equal amount of computational resources to establish its initial optimization potential \(\mathcal{P}_i^k\), which quantifies both the quantity and quality of POIs in the city. In subsequent rounds, each component is allocated a basic amount of computational resources \(\mathcal{I}_{\text{bas}}\), along with additional resources \(\mathcal{I}_{\text{add}}\) based on its optimization potential \(\mathcal{P}_i^k\), except for stagnant components. The optimization potential \(\mathcal{P}_i^k\) is then updated accordingly.

\subsection{Component Optimization and Assembly}

Each component is maintained by a subpopulation and evolves upon receiving the allocated computational resources. In this study, we adopt a variant of NSGA-II~\cite{deb2002fast} as the optimizer for each component, where the crossover and mutation operators are inspired by~\cite{huang2019automatic}. The details are summarized as follows:

\begin{itemize}
	\item \emph{Subpopulation Initialization:} Consider the $i$-th component with size \( d_iM \), where the corresponding set of POIs is \( V_i \). Each POI is assigned a unique identifier from \( 1 \) to \( |V_i| \), and each decision variable takes values from the set \( \{0,1,2,\dots, |V_i|\} \), where 0 indicates that the corresponding variable does not visit any POI. To initialize an individual, \( d_iM \) POIs are randomly selected from \( V_i \) and shuffled. Subsequently, each variable is set to 0 with a probability \( p_z \). This process is repeated until \( n \) individuals are generated, forming the initial subpopulation.
	\item \emph{Mutation:} Given an individual \( \mathbf{x} \ \), the mutation operator is then applied with probability \( p_m \), where a randomly selected variable \( x_j \) from \( \mathbf{x} \) is replaced with a randomly chosen element from the set \( \{0\} \cup (V_i \setminus \mathbf{x}) \), which represents the union of unvisited POIs and the placeholder 0.
	\item \emph{Crossover:} Two parent individuals are randomly selected from the population, and a two-point crossover is applied to generate two offspring. To ensure that no duplicate POIs exist within each offspring, we perform a post-crossover check on the exchanged segments of each offspring. If duplicate POIs are detected, they are corrected through element-wise mutation.
	\item \emph{Solution Assembly:} To estimate the contribution of each component, after obtaining a new subpopulation through iteration, individuals are assembled as segments into a replicated population constructed from the optimal solution, replacing the corresponding variables of the component. The global optimal solution is the candidate solution with the highest hypervolume contribution, and the hypervolume contribution is defined as the metric of the hypervolume formed by the solution's fitness value and the reference point as the diagonal.
\end{itemize}

\subsection{Overall Framework}

Algorithm \ref{alg:cccie} presents the complete procedure of DGCC. During initialization, the prescribed cluster partition of \(G=(V,E)\) defines the components. Theorem 1 provides a sufficient criterion for certifying weak decomposability when its graph and feasibility assumptions hold. For the multi-objective itinerary model, applying this guarantee to both travel time and travel cost requires the corresponding edge weights to satisfy the theorem's conditions and the path reorderings to preserve the daily and hotel-related constraints. With the visited vertices fixed, the collected POI score, node costs, and the visit-count-dependent factor \(\omega\) are unchanged. A visiting sequence for the components is then determined according to predefined travel preferences. The optimization potential of each component is initialized to 1, and the population \(P=\{P_1,\ldots,P_m\}\) is initialized to align with the decomposition.

In the optimization phase, each component is optimized in a cyclic manner, where the computational resources allocated to each subpopulation \( P_i \) are initially set to \( \mathcal{I}_{\text{bas}} +\mathcal{I}_{\text{add}}\). After each cycle, the optimization potential and available computational resources \( \mathcal{I}_{\text{avl}} \) for each component are updated. Additionally, the component sizes are dynamically adjusted every \(L\) optimization rounds. Finally, when the maximum number of evaluations is reached, the population \( P \) is returned as the obtained set of non-dominated solutions.
\begin{algorithm}[!h]
	\caption{DGCC}
	\label{alg:cccie}
	\SetAlgoLined
	
	Identify the prescribed cluster partition of graph \(G\);
	
	Initialize problem decomposition;
	
	Initialize the component visiting sequence according to travel preferences;
	
	Initialize the optimization potential \( \mathcal{P} = \text{ones}(m) \);
	
	Randomly initialize the population \( P = \{P_1, ..., P_m\} \);
	
	Initialize the optimization-round counter $k = 0$;
	
	\While{FEs \( < \) MaxFEs}{
		\For{$i = 1$ \textbf{to} $m$}{
			Optimize \( P_i \), with fitness evaluations \( \mathcal{I}_{\text{avl},i} \);
		}
		Update the optimization potential \( \mathcal{P} \) and available computational resources \( \mathcal{I}_{\text{avl}} \) for each component;
		
		$k \leftarrow k + 1$;
		
		\If{$\text{mod}(k, L) = 0$}{
			Dynamically adjust the component size (Alg. \ref{alg:Decomposition});
		}
	}
\end{algorithm}
\section{Experiments}
\label{sec:experiments}

In this section, we first introduce the experimental setup, including the real-world datasets, comparison algorithms, evaluation metric, and parameter settings. We then evaluate the proposed DGCC against state-of-the-art algorithms on real-world itinerary planning cases, large-scale semi-synthetic test cases, and adapted public COP benchmark cases. Finally, we conduct ablation, optimizer-decoupling, parameter-sensitivity, and qualitative analyses to evaluate the contributions of the core mechanisms, the general applicability of the proposed framework, its parameter behavior, and the quality of the generated itineraries. The DGCC code repository is hosted at \url{https://github.com/EvoNexusX/2026ZhangDGCC}.

\subsection{Experimental Design}
\noindent \textbf{Test Cases:} Attraction data were collected from the Ctrip website\footnote{[https://www.ctrip.com/]} covering 420 POIs across seven Chinese cities, 420 POIs across seven French cities, and 240 POIs across four German cities. The dataset includes location, popularity, recommended visit duration, ticket price, and other relevant information for each POI. Hotel data for each city were also sourced from Ctrip, while commuting information between locations was obtained via the Geoapify open platform API\footnote{[https://www.geoapify.com/]}.

The real-world test set comprises 18 test cases that vary in country, number of cities visited, and travel duration, as detailed in Table \ref{test}. The cities are indexed as follows in the subsequent analyses:

\begin{itemize}
	\item China (1–7): Nanjing (1), Changzhou (2), Lianyungang (3), Suzhou (4), Wuxi (5), Xuzhou (6), Yangzhou (7).
	\item France (8–14): Paris (8), Bordeaux (9), Lyon (10), Marseille (11), Nice (12), Mont Saint-Michel (13), Strasbourg (14).
	\item Germany (15–18): Frankfurt (15), Berlin (16), Munich (17), Cologne (18).
\end{itemize}

\begin{table}[h]
	\centering
	\caption{The description of the test cases}
	\label{test}
	\begin{tabular}{c c c c}
		\toprule
		Cases & Country & City Index & Travel Duration \\
		\midrule
		1 & China & 5,7 & 4 \\
		2 & China & 5,7 & 6 \\
		3 & France & 9,13 & 4 \\
		4 & France & 9,13 & 6 \\
		5 & Germany &16,18&4\\
		6 & Germany & 16,18&6\\
		7 & China & 3,6,2 & 6\\
		8 & China & 3,6,2 & 9\\
		9 & France & 10,11,14 & 6\\
		10 & France & 10,11,14 & 9\\
		11 & Germany &15,16,17 &6\\
		12 & Germany &15,16,17 &9\\
		13 & China & 1,4,5,7 & 8\\
		14 & China & 1,4,5,7 & 12\\
		15 & France & 8,10,12,13 & 8\\
		16 & France & 8,10,12,13 & 12\\
		17 & Germany & 15,16,17,18&8\\
		18 & Germany &15,16,17,18 &12\\
		\bottomrule
	\end{tabular}
\end{table}

\noindent \textbf{Objective Functions}: For each candidate itinerary \( \textbf{x} = ( x_1, \dots) \), we aim to minimize the travel time and travel cost, while maximizing the travel experience across three objectives.

\begin{equation}
	F(x) =
	\begin{cases}
		F_t(x) = \omega \sum\limits_{i=1}^{|\textbf{x}|-1} w_t(x_i, x_{i+1}), \\
		F_c(x) = \omega \left(\sum\limits_{i=1}^{|\textbf{x}|-1} w_c(x_i, x_{i+1}) + \sum\limits_{i=1}^{|\textbf{x}|} w_c(x_i) \right), \\
		F_{e}(x) = \frac{\theta}{\sum\limits_{i=1}^{|\textbf{x}|} w_e(x_i)}.
	\end{cases}
\end{equation}

Here, \( w_t(x_i, x_{i+1}) \) denotes the travel time between two POIs. \( w_c(x_i, x_{i+1}) \) represents the travel cost along the route, and the term \( w_c(x_i) \) accounts for the cost of visiting a POI, such as entrance fees. \( w_e(x_i) \) denotes the POI's score, which is derived from its popularity. To formulate this maximization objective as a minimization objective, we use the reciprocal of the total collected POI score.

The balancing factor \( \omega \) regulates the number of POIs in the itinerary, preventing the algorithm from excessively favoring either a minimal or excessive number of visits. It is defined as:

\begin{equation}\label{eqt:omega}
	\omega = 1 - \frac{k}{DM + \alpha}
\end{equation}
where \( D \) denotes the total number of travel days, \( M \) is the maximum number of POIs that can be visited per day, \( k \) is the number of POIs included in the solution, and \( \alpha > 0 \) is a control parameter.

Additionally, \( \theta \) is a scaling factor used to adjust the range of the objective value \( F_e \).

\noindent \textbf{Compared Algorithms and Parameter Setting:} We selected four metaheuristic optimization algorithms designed for itinerary planning as compared algorithms, including CCIP \cite{zhang2024cooperative}, AONSGA \cite{xu2021multi}, MOPSO \cite{yan2023multi}, and LSNSGA-II \cite{kolaee2023local}.	

The parameter settings for the comparison algorithms follow those in their original papers. The parameters for our proposed algorithm are as follows:  the maximum number of POIs visited per day \( M = 5 \), particle mutation probability \( p_m = 0.3 \), probability of setting a particle to zero during initialization \( p_z = 0.3 \), control parameter $\alpha=0.8$, and the population size for each component \( n = 100 \), the basic fitness evaluations $\mathcal{I}_{\text{bas}}$ and  the additional fitness evaluations $\mathcal{I}_{\text{add}}$ are set such that $\mathcal{I}_{\text{bas}} + \mathcal{I}_{\text{add}} = 20n$ and $\mathcal{I}_{\text{bas}} /   \mathcal{I}_{\text{add}} = 1$,scaling factor $\theta=10000$. The maximum number of evaluations for all algorithms is set to \( 30000m + 5000D \), where \( m \) is the number of cities and \( D \) is the total number of travel days. The algorithm terminates once the maximum number of evaluations is reached.

\noindent \textbf{Evaluation Metric:} The HV indicator evaluates the quality of the Pareto solution set obtained under the same fitness-evaluation budget, with a higher HV indicating better performance. A comparison group consists of the algorithms evaluated on the same test case and parameter configuration within a given experiment, including all repeated runs. For each group, a single reference point is determined from the componentwise worst objective values among the solutions obtained by all compared algorithms across all runs, and is then used for every HV calculation in that group. HV values are therefore compared within a group, rather than across groups with different reference points. Unless otherwise stated, results are reported over 50 independent runs; the semi-synthetic experiments in Section~\ref{sec:semi_synthetic} use 20 runs.

\subsection{Experimental Results}

\begin{table*}[ht]
	\centering
	\caption{The mean and standard deviation of HV obtained by DGCC and four compared algorithms on 18 test cases. }
	\label{tab:hv_result}
	\renewcommand{\arraystretch}{1.2} 
	\resizebox{0.97\textwidth}{!}{\begin{tabular}{lccccc}
			\toprule
			Cases & DGCC & CCIP & LSNSGA-II & MOPSO & AONSGA-II \\
			\midrule
			1  & \textbf{6.632e+07 ± 1.231e+06} & 6.251e+07 ±  1.276e+06 (++) & 5.793e+07  ±3.090e+06 (++) & 5.151e+07 ± 3.316e+06 (++) &4.180e+07 ±  1.645e+06 (++) \\
			2  & \textbf{1.266e+08 ±  2.896e+06} &  1.143e+08  ±  1.997e+06 (++) &  8.989e+07  ±  6.553e+06 (++) & 8.295e+07  ± 9.231e+06 (++) & 6.925e+07 ± 2.831e+06 (++) \\
			3  & \textbf{2.500e+10 ± 7.272e+08} & 1.920e+10 ± 3.089e+08 (++) & 1.906e+10 ± 1.599e+08 (++) & 1.857e+10 ± 1.044e+09 (++) & 9.125e+09 ± 8.209e+08 (++) \\
			4  & \textbf{6.279e+10 ± 8.828e+08} & 4.327e+10 ± 5.463e+08 (++) & 4.304e+10 ± 1.127e+09 (++) & 4.117e+10 ± 1.673e+09 (++) & 1.221e+10 ± 1.605e+09 (++) \\
			5  & \textbf{2.694e+10 ± 1.726e+08} & 2.398e+10 ± 1.385e+09 (++) & 2.220e+10 ± 2.172e+09 (++) & 1.908e+10 ± 2.153e+09 (++) & 1.319e+10 ± 1.567e+09 (++) \\
			6  & \textbf{1.030e+11 ± 2.161e+09} & 8.990e+10 ± 3.444e+09 (++) & 7.191e+10 ± 6.747e+09 (++) & 6.269e+10 ± 6.375e+09 (++) & 3.848e+10 ± 3.682e+09 (++) \\
			7  & \textbf{2.541e+09 ± 3.339e+07} & 2.339e+09 ± 4.956e+07 (++) & 2.059e+09 ± 1.390e+08 (++) & 1.803e+09 ± 1.656e+08 (++) & 1.492e+09 ± 5.747e+07 (++) \\
			8  & \textbf{8.871e+09 ± 1.485e+08} & 8.428e+09 ± 1.599e+08 (++) & 5.700e+09 ± 4.411e+08 (++) & 5.208e+09 ± 5.693e+08 (++) & 4.018e+09 ± 3.317e+08 (++) \\
			9  & \textbf{1.585e+11 ± 5.336e+09} & 1.292e+11 ± 8.245e+09 (++) & 1.064e+11 ± 7.204e+09 (++) & 8.687e+10 ± 7.204e+09 (++) & 3.032e+10 ± 7.204e+09 (++) \\
			10 & \textbf{4.547e+11 ± 2.383e+10} & 3.765e+11 ± 1.291e+10 (++) & 2.650e+11 ± 1.773e+10 (++) & 2.289e+11 ± 1.920e+10 (++) & 7.243e+10 ± 1.195e+10 (++) \\
			11 & \textbf{2.517e+11 ± 4.367e+09} & 2.307e+11 ± 9.628e+09 (++) & 1.961e+11 ± 1.621e+10 (++) & 1.669e+11 ± 1.696e+10 (++) & 9.774e+10 ± 9.001e+09 (++) \\
			12 & \textbf{6.756e+11 ± 1.220e+10} & 6.002e+11 ± 2.330e+10 (++) & 4.304e+11 ± 3.027e+10 (++) & 3.700e+11 ± 3.514e+10 (++) & 1.493e+11 ± 1.800e+10 (++) \\
			13 & \textbf{3.747e+08 ± 4.943e+06} & 3.412e+08 ± 5.874e+06 (++) & 2.833e+08 ± 2.139e+07 (++) & 2.601e+08 ± 1.774e+07 (++) & 1.955e+08 ± 1.865e+07 (++) \\
			14 & \textbf{8.343e+08 ± 1.275e+07} & 7.948e+08 ± 1.251e+07 (++) & 5.446e+08 ± 4.203e+07 (++) & 4.946e+08 ± 4.553e+07 (++) & 2.809e+08 ± 3.051e+07 (++) \\
			15 & \textbf{1.815e+11 ± 4.326e+09} & 1.515e+11 ± 6.209e+09 (++) & 1.523e+11 ± 8.994e+09 (++) & 1.440e+11 ± 8.284e+09 (++) & 4.682e+10 ± 4.879e+09 (++) \\
			16 & \textbf{8.271e+11 ± 3.034e+10} & 6.797e+11 ± 1.900e+10 (++) & 6.339e+11 ± 3.135e+10 (++) & 5.999e+11 ± 3.319e+10 (++) & 1.819e+11 ± 1.864e+10 (++) \\
			17 & \textbf{4.851e+11 ± 5.611e+09} & 4.221e+11 ± 2.512e+10 (++) & 3.317e+11 ± 2.264e+10 (++) & 2.699e+11 ± 3.338e+10 (++) & 1.112e+11 ± 1.359e+10 (++) \\
			18 & \textbf{1.169e+12 ± 2.355e+10} & 9.914e+11 ± 4.922e+10 (++) & 5.860e+11 ± 5.145e+10 (++) & 4.931e+11 ± 4.910e+10 (++) & 1.873e+11 ± 1.695e+10 (++) \\
			\midrule
			+/-/$\approx$ &  & 18/0/0 & 18/0/0 & 18/0/0 & 18/0/0 \\
			Rank & 1.000 & 2.056 & 2.944 & 4.000 & 5.000 \\
			\bottomrule
	\end{tabular}}
\end{table*}

Table \ref{tab:hv_result} presents the experimental results of the proposed DGCC and the four compared algorithms, where the best result for each test case is highlighted in bold.  To evaluate the overall statistical significance across the independent runs of all compared algorithms, the Kruskal-Wallis test is conducted. Subsequently, statistical significance between pairs of algorithms is assessed using the independent Wilcoxon rank-sum test. The Holm-Bonferroni method is applied as a post hoc correction to control the family-wise error rate. Here, \texttt{++}/\texttt{- -} denote highly significant improvements/deteriorations (adjusted $p < 0.01$), \texttt{+}/\texttt{-} indicate moderate significance ($0.01 \leq \text{adjusted } p < 0.05$), and $\approx$ represents no statistically significant difference (adjusted $p \geq 0.05$ or negligible mean difference). Overall, the results demonstrate that DGCC outperforms all other algorithms across all test cases, indicating its strong competitiveness. Specifically, both DGCC and CCIP employ a divide-and-conquer-based CC framework, achieving superior performance compared to LSNSGA-II, MOPSO, and AONSGA-II. Furthermore, within the same country, the performance improvement of DGCC becomes more pronounced as the number of cities in the test case increases. For instance, in the six test cases for China, i.e., cases 1, 2 (two cities), cases 7, 8 (three cities), and cases 13, 14 (four cities), the advantage of the CC framework becomes increasingly evident, further supporting its effectiveness in large-scale itinerary planning. Considering test cases within the same country and city but with varying travel durations, we observe that DGCC exhibits greater performance gains for longer travel durations. Since our encoding scheme closely links travel duration to problem dimensionality, this result further validates the scalability of DGCC. Finally, when comparing the same algorithm across different countries, we note significant variations in performance magnitudes. This suggests that results across countries are not directly comparable, as factors such as the geographical distribution of POIs influence the outcomes.

\subsection{Evaluation on Large-scale Semi-synthetic Test Set}
\label{sec:semi_synthetic}
In the previous section, we verified the practicality and effectiveness of the DGCC algorithm in real-world itinerary planning scenarios. However, existing classic combinatorial optimization datasets—such as TSPLIB \cite{reinelt1991tsplib}, which is widely used for TSP problems, and the National TSP Collection\footnote{Available at: \url{https://www.math.uwaterloo.ca/tsp/world/countries.html}} for large-scale optimization—often struggle to simultaneously satisfy composite requirements such as multi-objective cost evaluation, massive micro-variables, and a strict cluster-like distribution structure. Therefore, to further explore the theoretical performance of the algorithm under extreme scales, this section introduces a large-scale semi-synthetic benchmark suite designed to provide a controlled experimental environment with extreme variable scales and rigorous topological boundaries.\\
The synthetic test suite consists of four independent national-scale test instances. These four instances are constructed based on the real geographical topologies of China, Japan, the United States, and the United Kingdom, respectively. Each instance includes several core cities, with POI scales ranging from 15,000 to nearly 40,000. The dataset exhibits a significant cluster structure in terms of spatial topology. Regarding the relative positioning of the clusters, the center coordinates of each cluster (i.e., macro-city) on the 2D plane are directly mapped from the geographic longitude and latitude of the corresponding real-world core cities to restore the relative positions of the cities as much as possible. Each cluster is defined as a regular circular area, with its boundary radius partially referencing the actual built-up area of the city. The POIs are scattered within the defined circular boundaries using a constrained uniform distribution. Simultaneously, to reflect the characteristic differences between different regions in reality, each cluster is assigned an independent base average node cost \(W_i\), while the cost of an individual POI within the cluster fluctuates randomly within the interval $[0.5W_i, 1.5W_i]$. Based on the aforementioned construction logic, Figure \ref{large_scale_datasets} provides detailed schematic diagrams and specific parameter configurations for these four sets of test instances. The floating statistical boxes in the figure summarize the total number of cities and the POI scale for each dataset; meanwhile, the local labels for each cluster precisely indicate the number of POIs within the city, the average node cost (Avg W), and the actual search radius (R).\\
For this test suite, we have set two objective functions: minimization of total travel distance and total node cost. Assuming the effective visit sequence of the current individual is $X=(x_1, x_2, \dots, x_n)$, the objective functions are:
$$
F(x) =  \begin{cases}    F_1(X) = \sum_{i=1}^{n-1} d(x_i, x_{i+1}), \\    F_2(X) = \sum_{i=1}^n w(x_i).  \end{cases}
$$
where $d(x_i, x_{i+1})$ is the Euclidean distance between points $x_i$ and $x_{i+1}$, and $w(x_i)$ is the non-negative cost assigned to node $x_i$. Furthermore, the solution sequence $X = (x_1, x_2, \dots, x_n)$ must satisfy the following constraints, otherwise, a large penalty value (set to $10^{10}$ in this experiment) will be imposed during fitness evaluation:\\
(1) The sequence $X$ must not contain duplicate points;\\
(2) The sequence $X$ must include at least one vertex from every cluster;\\
Based on the four instances of this test suite, we have established two experimental configurations, with specific parameters as shown in Table \ref{large_para}. Regarding the algorithm parameter settings, $p_z$ is set to 0, while all other parameters remain unchanged. For each group of experiments, 20 rounds of parallel tests were conducted, and the results for each group represent the mean and standard deviation of the HV across the 20 runs.\\
Table \ref{large_result} presents the experimental results of the HV metric for DGCC and the four comparison algorithms across eight semi-synthetic test cases, where the best result for each test case is highlighted in bold. To evaluate the overall statistical significance of all comparison algorithms across independent runs, a Kruskal-Wallis test was first conducted. Subsequently, an independent-sample Wilcoxon rank-sum test was employed for post-hoc evaluation of statistical significance between pairs of algorithms, and the Holm-Bonferroni method was applied as a post-hoc correction to control the family-wise error rate. As shown in the table, (++) indicates that the overall performance of the comparison algorithm remains extremely significantly inferior to DGCC after multiple comparison corrections (adjusted p < 0.01). Overall, DGCC achieved the highest mean HV in all eight large-scale test cases, demonstrating a substantial advantage over the other algorithms.\\
Through the examination of the algorithmic results, we found that the failure of the non-decomposition-based algorithms (MOPSO, LSNSGA-II, and AONSGA-II) did not stem from insufficient node coverage, but rather from the disruption of the macro-spatial topological structure by global operators. All results were able to cover all clusters without triggering penalties. However, all solutions exhibited ``interleaved visits,'' where a previously visited cluster was revisited after visiting other clusters. This occurs primarily because the search capability of traditional algorithms is relatively low, making them prone to falling into local optima in large-scale optimization. It is difficult to find the correct visitation rules within limited optimization resources, leading to interleaved visits in all solutions. While CCIP also adopts a decomposition scheme and avoids interleaved visits, it predefines a static sub-problem decomposition structure during the initialization phase. In such large-scale scenarios, it is difficult to know the optimal structure in advance. Furthermore, in multi-objective problem models, non-dominated solutions do not necessarily all originate from a single structure. Consequently, CCIP remains significantly inferior to DGCC, which has the capability to dynamically adjust its structure.
\begin{figure*}[htbp]
	\centering
	
	\begin{minipage}[b]{0.48\textwidth}
		\centering
		\includegraphics[width=\linewidth]{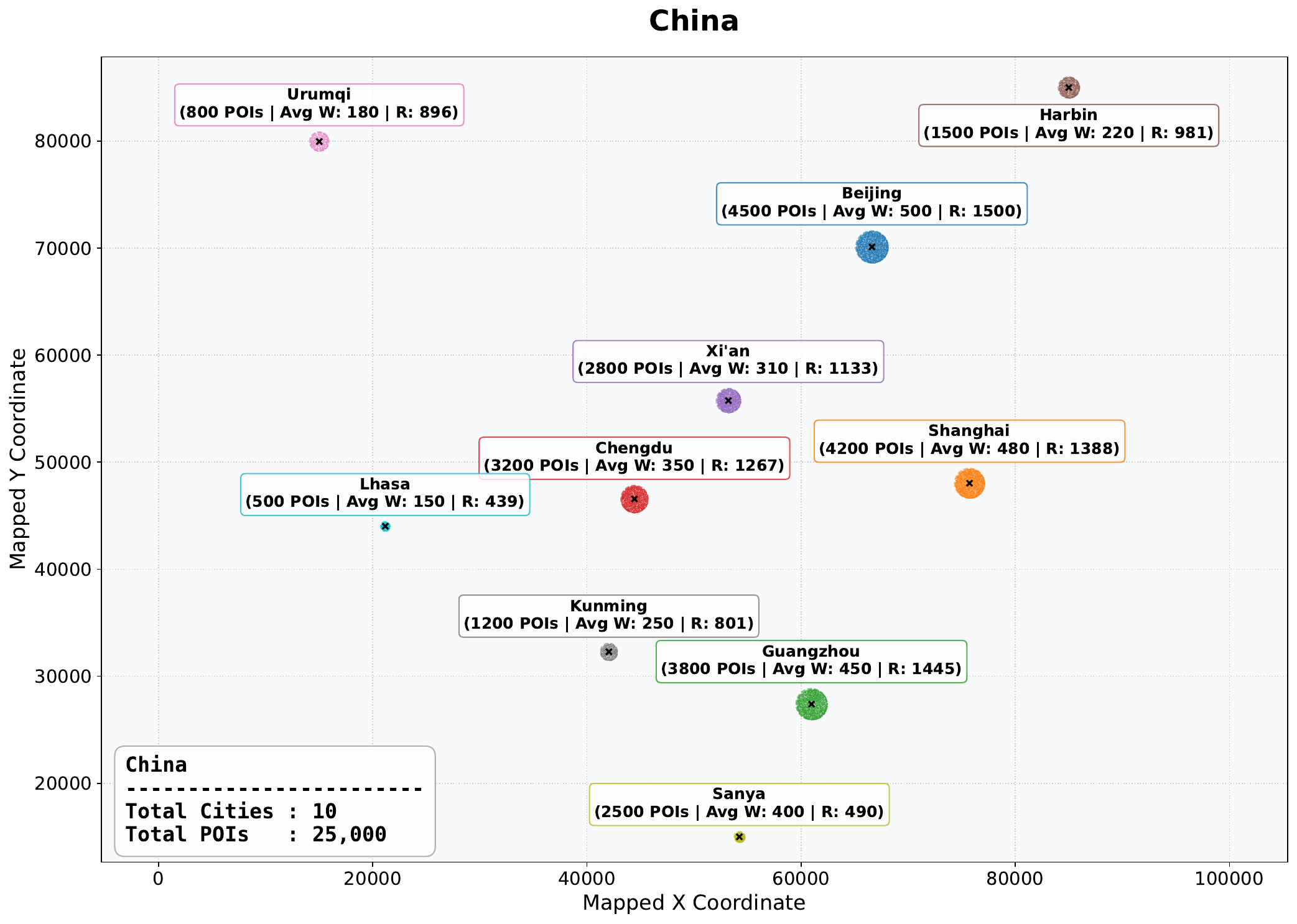}
		\vspace{0.1cm} \\ 
		(a)
	\end{minipage}
	\hfill
	\begin{minipage}[b]{0.48\textwidth}
		\centering
		\includegraphics[width=\linewidth]{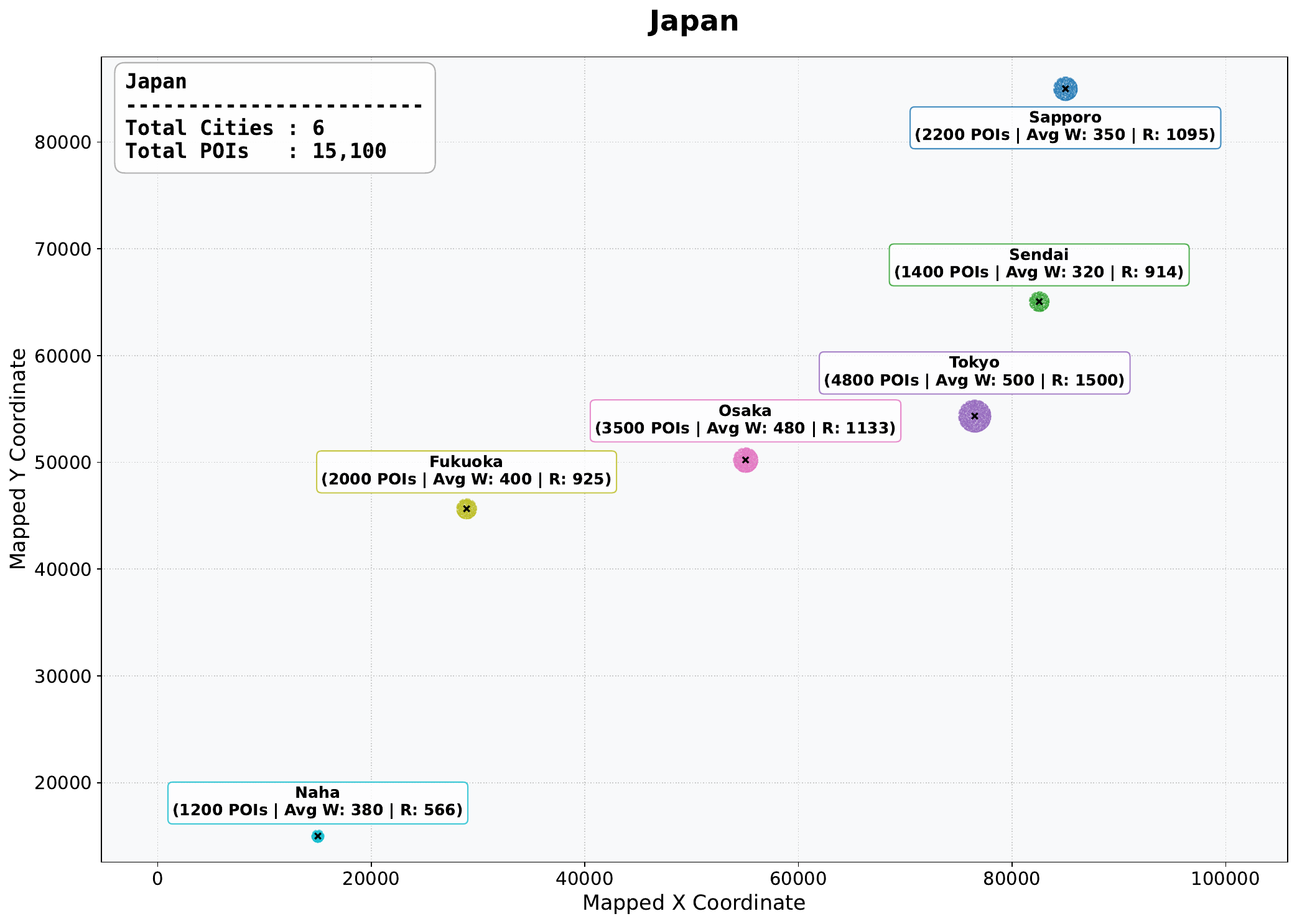}
		\vspace{0.1cm} \\
		(b)
	\end{minipage}
	
	\vspace{0.4cm} 
	
	\begin{minipage}[b]{0.48\textwidth}
		\centering
		\includegraphics[width=\linewidth]{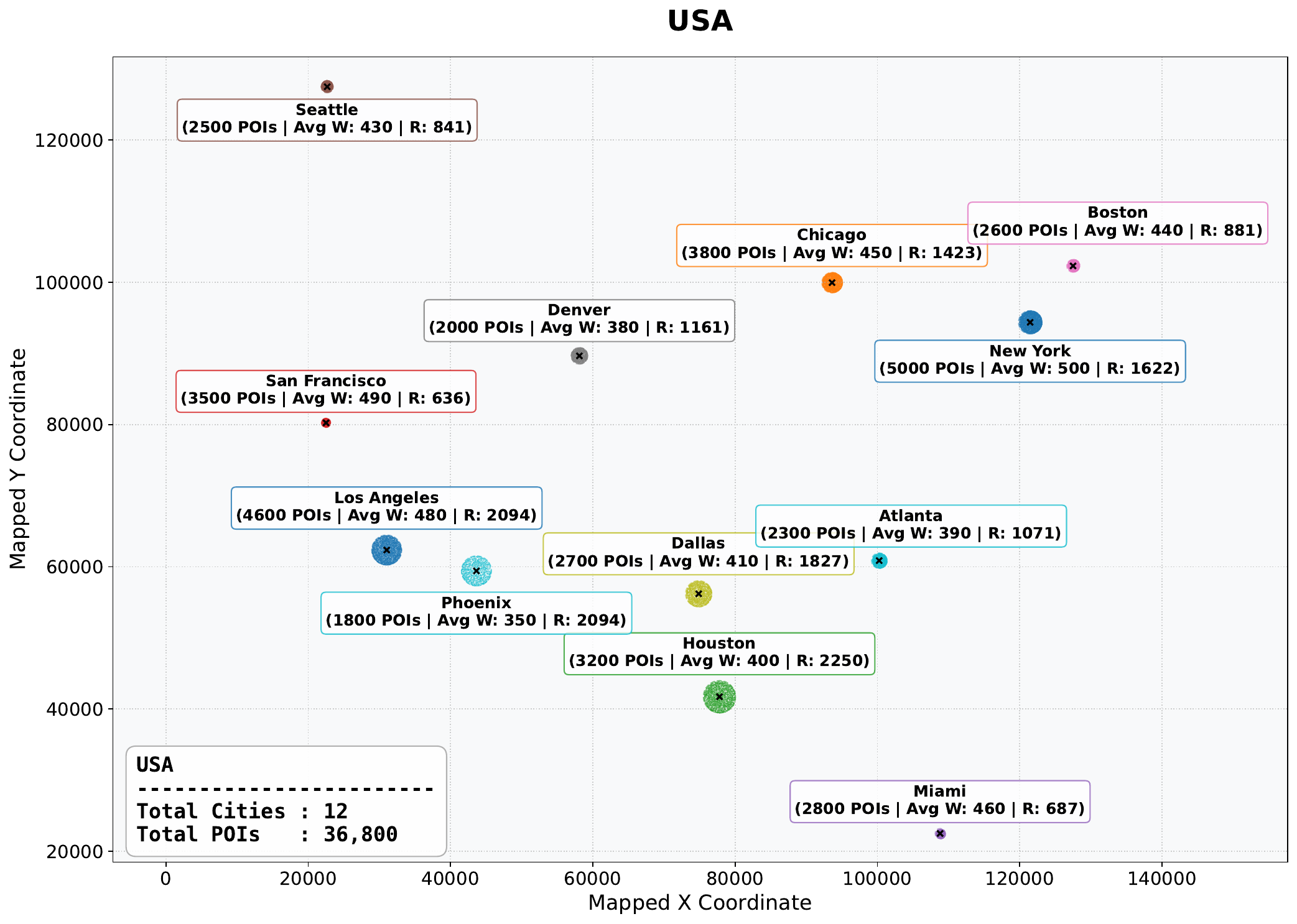}
		\vspace{0.1cm} \\
		(c)
	\end{minipage}
	\hfill
	\begin{minipage}[b]{0.48\textwidth}
		\centering
		\includegraphics[width=\linewidth]{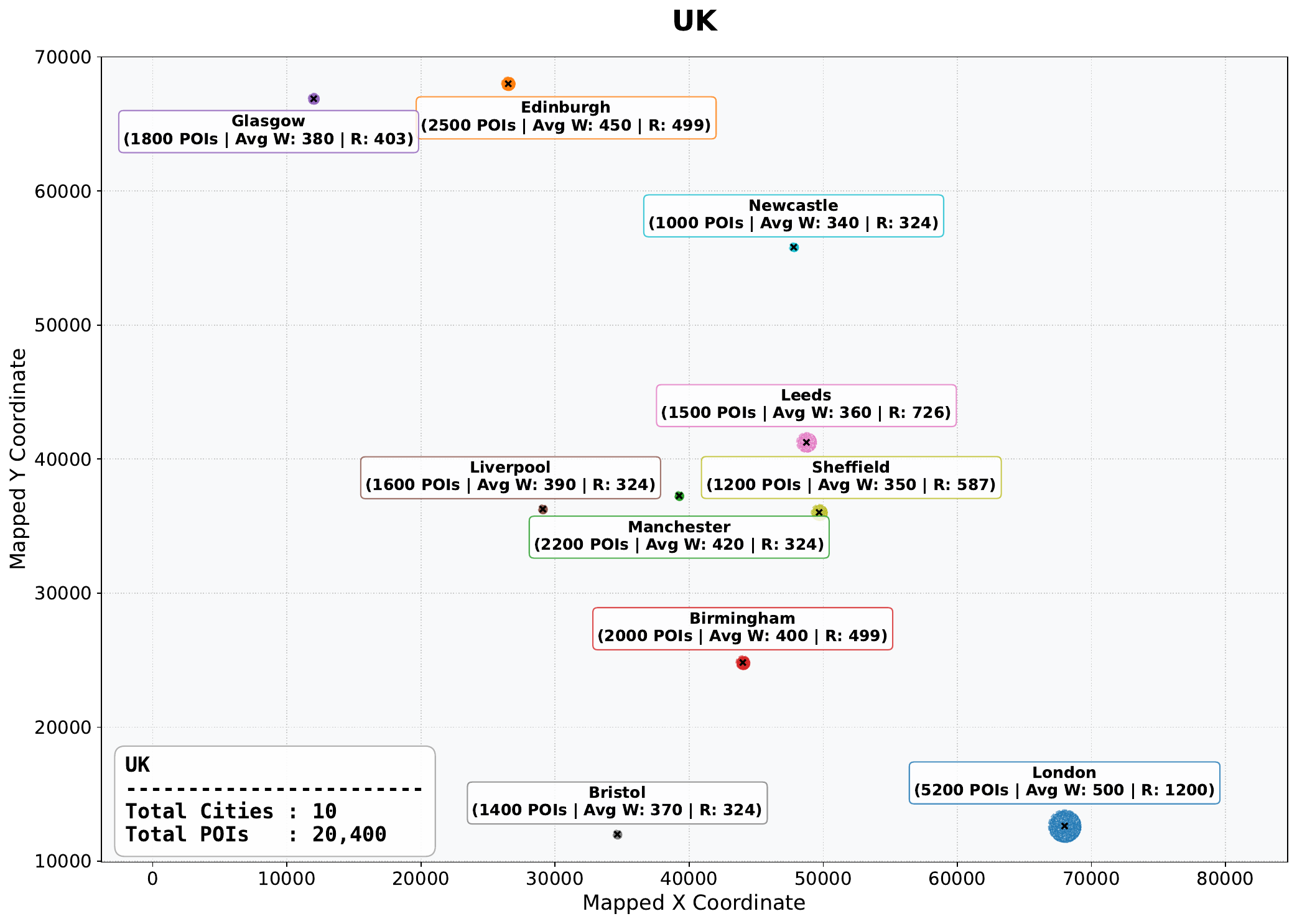}
		\vspace{0.1cm} \\
		(d)
	\end{minipage}
	
	\caption{Four test set instance parameters}
	\Description{Four panels show the geographical cluster configurations for the China, Japan, USA, and UK semi-synthetic test sets. Annotations report cluster attributes and the total numbers of cities and points of interest.}
	\label{large_scale_datasets}
\end{figure*}

\begin{table}[h]
	\centering
	\caption{Parameter settings for the semi-synthetic dataset.}
	\label{large_para}
	\begin{tabular}{c c c c c}
		\toprule
		Cases & Country & No. of Cities & Solution Dimension & Max FEs \\
		\midrule
		1 & China & 10 & 500 & 300000 \\
		2 & China & 10 & 1000 & 500000 \\
		3 & Japan & 6  & 300  & 180000 \\
		4 & Japan & 6  & 600 & 300000 \\
		5 & USA   & 12 & 600 & 360000 \\
		6 & USA   & 12 & 1200 & 600000 \\
		7 & UK    & 10 & 500 & 300000 \\
		8 & UK    & 10 & 1000 & 500000 \\
		\bottomrule
	\end{tabular}
\end{table}

\begin{table*}[ht]
	\centering
	\caption{The mean and standard deviation of HV obtained by DGCC and four compared algorithms on 8 test cases. }
	\label{large_result}
	\renewcommand{\arraystretch}{1.2} 
	\resizebox{0.97\textwidth}{!}{\begin{tabular}{lccccc}
			\toprule
			Cases & DGCC & CCIP & LSNSGA-II & MOPSO & AONSGA-II \\
			\midrule
			1  & \textbf{1.620e+12 ± 3.973e+10} & 1.536e+12 ± 1.173e+10 (++) & 4.556e+11 ± 2.433e+10 (++) & 8.136e+11 ± 2.199e+10 (++) & 5.199e+10 ± 2.963e+09 (++) \\
			2  & \textbf{5.965e+11 ± 1.657e+10} & 5.625e+11 ± 4.491e+09 (++) & 1.204e+11 ± 1.122e+10 (++) & 1.640e+11 ± 1.417e+10 (++) & 2.294e+10 ± 1.632e+09 (++) \\
			3  & \textbf{3.928e+12 ± 1.247e+11} & 3.722e+12 ± 3.648e+10 (++) & 5.161e+11 ± 3.689e+10 (++) & 6.399e+11 ± 4.917e+10 (++) & 7.653e+10 ± 4.253e+09 (++) \\
			4  & \textbf{1.543e+12 ± 4.202e+10} & 1.459e+12 ± 1.623e+10 (++) & 2.585e+11 ± 1.807e+10 (++) & 3.268e+11 ± 2.367e+10 (++) & 4.563e+10 ± 3.695e+09 (++) \\
			5  & \textbf{6.179e+12 ± 1.024e+11} & 5.793e+12 ± 3.260e+10 (++) & 6.008e+11 ± 4.019e+10 (++) & 8.017e+11 ± 7.077e+10 (++) & 8.588e+10 ± 5.160e+09 (++) \\
			6  & \textbf{2.279e+12 ± 3.002e+10} & 2.142e+12 ± 1.884e+10 (++) & 3.403e+11 ± 1.897e+10 (++) & 4.427e+11 ± 3.533e+10 (++) & 6.639e+10 ± 7.583e+09 (++) \\
			7  & \textbf{1.515e+13 ± 2.439e+11} & 1.401e+13 ± 1.085e+11 (++) & 1.491e+12 ± 9.357e+10 (++) & 1.753e+12 ± 1.464e+11 (++) & 1.049e+11 ± 8.472e+09 (++) \\
			8  & \textbf{5.709e+12 ± 9.746e+10} & 5.279e+12 ± 5.034e+10 (++) & 6.735e+11 ± 3.630e+10 (++) & 8.422e+11 ± 5.288e+10 (++) & 4.845e+10 ± 6.124e+09 (++) \\
			\midrule
			+/-/$\approx$ &  & 8/0/0 & 8/0/0 & 8/0/0 & 8/0/0 \\
			Rank & 1.000 & 2.000 & 4.000 & 3.000 & 5.000 \\
			\bottomrule
	\end{tabular}}
\end{table*}

\subsection{Evaluation on Public Benchmark Datasets}
We referred to the clustered orienteering problem (COP) benchmark from \cite{angelelli2014clustered}\footnote{Available at: \url{http://or-brescia.unibs.it/}}, and selected the eight largest datasets from this benchmark for testing. Since the original COP is a single-objective profit-maximization problem, we adapt these instances to a bi-objective minimization setting while retaining the vertex coordinates and cluster partitions. This problem can be represented as an undirected complete graph $G = (V, E)$, where $V = \{v_1, v_2, \ldots, v_n\}$ is the set of all vertices. The vertex set is partitioned into $m$ disjoint clusters $S = \{S_1, S_2, \ldots, S_m\}$, satisfying $\cup_{c=1}^m S_c = V$. Each vertex $v_i$ has a spatial coordinate (used to calculate the Euclidean distance $d(v_i, v_j)$ between vertices) and a non-negative node cost $w_i$. Since node costs are not provided for these instances, we generate them as follows. To introduce differences in node costs across clusters, we set the average cost of the \(c\)-th cluster to \(c\times100\). The cost of each node within the cluster is then sampled uniformly from 0.5 to 1.5 times this average. A feasible solution to the problem is denoted as a vertex sequence $X = (x_1, x_2, \ldots, x_k)$, where $x_i \in V$ represents the $i$-th visited vertex in the route, and $k$ is the total number of vertices included in the solution, representing the dimension of the solution. The objective functions are defined as follows:

$$
F(x) =
\begin{cases}
	F_1(X) = \sum_{i=1}^{k-1} d(x_i, x_{i+1}), \\
	F_2(X) = \sum_{i=1}^k w(x_i).
\end{cases}
$$

Additionally, the sequence of solutions $X = (x_1, x_2, \ldots, x_k)$ must satisfy the following constraints; otherwise, a substantial penalty value will be imposed during fitness evaluation. The constraints are as follows:\\
(1) Sequence $X$ must not contain duplicate vertices;\\
(2) Sequence $X$ must include at least one vertex from each cluster;\\
(3) Cross-visiting is prohibited: for all vertices in sequence $X$ belonging to the same cluster, their indices in the sequence must be contiguous.\\
We conducted tests on eight sets of experimental configurations, specifically att532s10 and att532s15 from the COP benchmark. These two groups of tests contain 532 data points and are partitioned into 10 and 15 clusters, respectively. The experimental parameter settings are shown in Table \ref{tab:COP_para}. Regarding the algorithm parameter settings, all parameters remain unchanged except for $p_z$, which is set to 0.\\
Table \ref{tab:hv_result_new} presents the comparative experimental results between DGCC and other algorithms, where the best result for each test case is highlighted in bold. The Kruskal-Wallis test and the independent Wilcoxon rank-sum test with Holm-Bonferroni correction were employed to assess statistical significance. Overall, the experimental results clearly demonstrate that DGCC achieved the best performance across all eight test cases, and the adjusted $p$-values indicate that it is significantly superior to all compared algorithms. CCIP, which also adopts a divide-and-conquer approach, significantly outperformed algorithms that do not utilize decomposition strategies, such as LSNSGA-II. These experimental characteristics exhibit a high degree of consistency with the algorithmic performance features observed in tests on real-world data, indicating that the strategy possesses strong generalization capabilities.

\begin{table}[h]
	\centering
	\caption{Parameter settings for the COP test set.}
	\label{tab:COP_para}
	\begin{tabular}{c c c c c}
		\toprule
		Cases & Instance & No. of Clusters & Solution Dimension $k$ & Max FEs \\
		\midrule
		1 & att532s10g1q2.cop & 10 & 250 & 200000 \\
		2 & att532s10g1q3.cop & 10 & 250 & 200000 \\
		3 & att532s10g2q2.cop & 10 & 250 & 200000 \\
		4 & att532s10g2q3.cop & 10 & 250 & 200000 \\
		5 & att532s15g1q2.cop & 15 & 300 & 270000 \\
		6 & att532s15g1q3.cop & 15 & 300 & 270000 \\
		7 & att532s15g2q2.cop & 15 & 300 & 270000 \\
		8 & att532s15g2q3.cop & 15 & 300 & 270000 \\
		\bottomrule
	\end{tabular}
\end{table}

\begin{table*}[ht]
	\centering
	\caption{The mean and standard deviation of HV obtained by DGCC and four compared algorithms on 8 test cases. }
	\label{tab:hv_result_new}
	\renewcommand{\arraystretch}{1.2} 
	\resizebox{0.97\textwidth}{!}{\begin{tabular}{lccccc}
			\toprule
			Cases & DGCC & CCIP & LSNSGA-II & MOPSO & AONSGA-II \\
			\midrule
			1  & \textbf{1.063e+14 ± 1.383e+11} & 1.058e+14 ± 1.568e+11 (++) & 2.793e+10 ± 9.778e+08 (++) & 3.264e+10 ± 9.388e+08 (++) & 2.149e+09 ± 2.479e+08 (++) \\
			2  & \textbf{1.065e+14 ± 1.482e+11} & 1.060e+14 ± 9.621e+10 (++) & 3.185e+10 ± 1.145e+09 (++) & 3.698e+10 ± 8.812e+08 (++) & 3.726e+09 ± 3.931e+08 (++) \\
			3  & \textbf{1.062e+14 ± 1.104e+11} & 1.057e+14 ± 1.448e+11 (++) & 2.718e+10 ± 8.937e+08 (++) & 3.166e+10 ± 9.840e+08 (++) & 1.786e+09 ± 2.405e+08 (++) \\
			4  & \textbf{1.064e+14 ± 1.325e+11} & 1.058e+14 ± 1.362e+11 (++) & 2.799e+10 ± 1.257e+09 (++) & 3.297e+10 ± 7.951e+08 (++) & 2.016e+09 ± 2.339e+08 (++) \\
			5  & \textbf{1.079e+14 ± 2.967e+11} & 1.069e+14 ± 1.947e+11 (++) & 5.763e+10 ± 1.477e+09 (++) & 6.214e+10 ± 1.472e+09 (++) & 4.453e+09 ± 5.725e+08 (++) \\
			6  & \textbf{1.080e+14 ± 2.303e+11} & 1.069e+14 ± 1.605e+11 (++) & 5.516e+10 ± 2.015e+09 (++) & 6.008e+10 ± 1.360e+09 (++) & 3.935e+09 ± 4.618e+08 (++) \\
			7  & \textbf{1.080e+14 ± 2.345e+11} & 1.068e+14 ± 1.785e+11 (++) & 5.208e+10 ± 1.712e+09 (++) & 5.702e+10 ± 1.497e+09 (++) & 2.967e+09 ± 4.964e+08 (++) \\
			8  & \textbf{1.075e+14 ± 1.830e+11} & 1.066e+14 ± 1.946e+11 (++) & 5.378e+10 ± 1.785e+09 (++) & 5.769e+10 ± 1.612e+09 (++) & 3.186e+09 ± 2.646e+08 (++) \\
			\midrule
			+/-/$\approx$ &  & 8/0/0 & 8/0/0 & 8/0/0 & 8/0/0 \\
			Rank & 1.000 & 2.000 & 4.000 & 3.000 & 5.000 \\
			\bottomrule
	\end{tabular}}
\end{table*}

\subsection{Ablation Study}
DGCC consists of three core parts, i.e., component structure dynamic adjustment, computational resource allocation, and population inheritance. In this section, we conduct an ablation study to analyze the impact of each part on the algorithm’s performance. The experimental results are presented in Table \ref{tab:ablation_hv}.
\begin{itemize}
	\item \emph{w/o structure adjustment} refers to the version of DGCC that does not utilize the component structure dynamic adjustment strategy. In this case, the component structure remains fixed throughout the evolutionary process, with variables uniformly distributed across all components, and the structure of the components remains identical.
	\item \emph{w/o resource allocation} refers to the version where the dynamic computational resource allocation strategy is not applied. Here, each component is assigned an equal amount of computational resources.
	\item \emph{w/o population inheritance} refers to the version of DGCC that does not use the population inheritance strategy. In this case, after each iteration of the main loop, new component populations are initialized randomly.
\end{itemize}

\begin{table*}[h]
	\centering
	\caption{The mean and standard deviation of HV about the ablation study of DGCC on 18 test cases.}
	\label{tab:ablation_hv}
	\renewcommand{\arraystretch}{1.2}
	\resizebox{0.97\textwidth}{!}{\begin{tabular}{c c c c c}
			\toprule
			Case & DGCC & w/o structure adjustment & w/o resource allocation & w/o population inheritance \\
			\midrule
			1 & \textbf{1.812e+06 ± 4.904e+04} & 1.761e+06 ± 4.137e+04 & 1.809e+06 ± 6.137e+04 & 1.396e+06 ± 5.886e+04 \\
			2  & \textbf{4.736e+06 ± 2.151e+05} & 4.525e+06 ± 1.824e+05 & 4.705e+06 ± 1.995e+05 & 3.040e+06 ± 2.209e+05 \\
			3  & \textbf{5.963e+09 ± 2.292e+08} & 4.537e+09 ± 3.976e+08 & 5.840e+09 ± 2.825e+08 & 4.148e+09 ± 2.450e+08 \\
			4  & \textbf{2.150e+10 ± 4.201e+08} & 1.647e+10 ± 7.287e+08 & 2.111e+10 ± 4.707e+08 & 1.420e+10 ± 4.289e+08 \\
			5  & \textbf{1.421e+09 ± 5.328e+07} & 1.411e+09 ± 7.589e+07 & 1.409e+09 ± 5.872e+07 & 1.221e+09 ± 1.083e+08 \\
			6  & \textbf{7.829e+09 ± 3.546e+08} & 7.451e+09 ± 6.230e+08 & 7.809e+09 ± 3.583e+08 & 5.612e+09 ± 7.196e+08 \\
			7  & \textbf{5.872e+08 ± 2.878e+07} & 5.827e+08 ± 1.834e+07 & 5.815e+08 ± 2.618e+07 & 4.087e+08 ± 1.725e+07 \\
			8  & \textbf{1.178e+09 ± 6.060e+07} & 1.165e+09 ± 6.007e+07 & 1.164e+09 ± 5.199e+07 & 7.082e+08 ± 5.086e+07 \\
			9  & \textbf{2.053e+10 ± 1.042e+09} & 1.884e+10 ± 8.867e+08 & 1.933e+10 ± 1.270e+09 & 1.311e+10 ± 1.479e+09 \\
			10 & \textbf{8.364e+10 ± 5.567e+09} & 8.045e+10 ± 4.500e+09 & 8.234e+10 ± 6.173e+09 & 4.198e+10 ± 4.980e+09 \\
			11 & \textbf{3.988e+09 ± 1.264e+08} & 3.917e+09 ± 2.430e+08 & 3.968e+09 ± 1.310e+08 & 2.923e+09 ± 3.622e+08 \\
			12 & \textbf{2.873e+10 ± 1.302e+09} & 2.624e+10 ± 2.255e+09 & 2.747e+10 ± 1.923e+09 & 1.523e+10 ± 2.253e+09 \\
			13 & \textbf{8.580e+06 ± 5.996e+05} & 8.399e+06 ± 2.837e+05 & 8.452e+06 ± 2.047e+05 & 5.256e+06 ± 2.921e+05 \\
			14 & \textbf{1.550e+07 ± 6.906e+05} & 1.501e+07 ± 6.101e+05 & 1.521e+07 ± 6.948e+05 & 8.045e+06 ± 4.706e+05 \\
			15 & \textbf{1.435e+10 ± 8.661e+08} & 1.281e+10 ± 7.763e+08 & 1.370e+10 ± 1.179e+09 & 1.032e+10 ± 6.303e+08 \\
			16 & \textbf{5.130e+10 ± 4.493e+09} & 4.439e+10 ± 1.701e+09 & 4.821e+10 ± 4.093e+09 & 2.804e+10 ± 2.895e+09 \\
			17 & \textbf{1.148e+10 ± 4.258e+08} & 1.130e+10 ± 4.272e+08 & 1.133e+10 ± 4.655e+08 & 8.512e+09 ± 8.794e+08 \\
			18 & \textbf{6.706e+10 ± 4.022e+09} & 6.446e+10 ± 3.882e+09 & 6.540e+10 ± 3.389e+09 & 3.368e+10 ± 5.373e+09 \\
			\bottomrule
	\end{tabular}}
\end{table*}

Overall, the comparison experimental results indicate that DGCC outperforms the other three algorithms in all test cases, demonstrating that each component contributes to the overall performance. Specifically, the population inheritance mechanism has the greatest impact, followed by the component structure adjustment and computational resource allocation strategies.

The experimental results without population inheritance are the poorest. This is because, after the dynamic adjustment of the component structure, randomly restarting the subpopulations results in the loss of a significant amount of historical information, leading to a waste of many evaluations. By employing the population inheritance strategy and retaining the subpopulations that have not been adjusted, the loss of optimization knowledge is effectively minimized, thereby improving overall optimization performance.

The impacts of DGCC without structure adjustment and without resource allocation are similar, as these two strategies are interdependent. On the one hand, when the component structure is identical, the contribution of each component is insufficient to reflect the differences in component attributes. Specifically, although there are differences in tourism resources across cities, the identical component structure imposes an upper limit on the number of visitable POIs in each city. The primary differences in tourism resources between cities are sources from the number of POIs, while the differences in the quality of representative POIs within each city are relatively small. However, this difference in POI quality is the main source of component contribution quantification in DGCC without structure adjustment. The inaccuracy in contribution assessment also renders the subsequent computational resource allocation strategy inefficient.

On the other hand, when the computational resources allocated to the components are uniform, the structural adjustments of the components are conservative. First, high-value components are not allocated enough computational resources to optimize effectively, which hampers the algorithm's ability to assess their value and delays the structural adjustment of the components. Second, after the structure of the components is adjusted, if the computational resources allocated to them remain the same as those allocated to other components, the evaluation of the contribution of each basic encoding unit becomes inaccurate. For instance, when the number of variables assigned to a component increases, the contribution of each basic encoding unit may decrease due to the increased optimization difficulty, leading to erroneous decisions in subsequent iterations.

\subsection{Decoupling Analysis of Framework and Optimizers}

To further verify the rationality and generalization capability of the DGCC framework, this section primarily analyzes the decoupling relationship between the DGCC framework and the optimizers. Specifically, we embedded three optimizers—NSGA-II, MMTSP-DS~\cite{yue2026evolutionary}, and HI-MOEA \cite{ding2024history}—into the DGCC framework and compared them with their standalone versions operating independently of the framework. The experiment covers four sets of test cases each for the itinerary planning test cases and the COP benchmark test cases; the mapping relationship of test case IDs is detailed in Table \ref{tab:cases_mapping}. The parameters were set in accordance with the previous experimental design.
Table \ref{tab:hv_ablation_pairwise} details the mean and standard deviation of the hypervolume (HV) for the three different underlying operators with and without the DGCC framework. The significance test markers in the table represent pairwise comparisons conducted within the same optimizer. Groups 1–4 correspond to the itinerary planning test sets, while groups 5–8 correspond to the test sets from the COP benchmark.
By longitudinally observing the results of the same optimizer under different configurations, it is evident that the introduction of the DGCC framework exerts a decisive enhancement effect. In all test cases, the algorithmic versions equipped with the DGCC framework outperformed their corresponding framework-less standalone versions by a substantial margin, and the adjusted $p$-values from the independent Wilcoxon rank-sum test (corrected via the Holm-Bonferroni method) indicate that all improvements are highly significant. In the larger-scale COP test section, this improvement is even more pronounced, with a difference in HV magnitude of approximately five orders of magnitude. This suggests that the DGCC framework generally possesses a greater advantage when dealing with larger-scale test cases. Furthermore, although Table \ref{tab:hv_ablation_pairwise} only marks the significance within homologous algorithms, empirical statistical tests show that any algorithm utilizing the DGCC framework significantly outperforms all other standalone optimizers without the framework. This indicates that when facing the complex instances in this study, possessing a macro-scheduling framework like DGCC is more critical than the relative advancement of the underlying optimizer itself. For addressing other problems, the plug-and-play nature of DGCC also allows the underlying optimization operators to be customized according to specific problem requirements and then integrated into the DGCC framework to achieve superior performance.

\begin{table}[ht]
	\centering
	\caption{Mapping relationship between the evaluated test cases and the original datasets.}
	\label{tab:cases_mapping}
	\renewcommand{\arraystretch}{1.2} 
	\resizebox{0.95\textwidth}{!}{
		\begin{tabular}{l *{8}{>{\centering\arraybackslash}p{1.2cm}}}
			\toprule
			\multirow{2}{*}{Test Sets} & \multicolumn{4}{c}{Dataset 1 (Itinerary Planning)} & \multicolumn{4}{c}{Dataset 2 (COP Instances)} \\
			\cmidrule(lr){2-5} \cmidrule(lr){6-9}
			Original Index & 7 & 8 & 13 & 14 & 1 & 2 & 5 & 6 \\
			\midrule
			New Cases & 1& 2 & 3 & 4 & 5 & 6 & 7 & 8 \\
			\bottomrule
		\end{tabular}
	}
\end{table}

\begin{table*}[ht]
	\centering
	\caption{The mean and standard deviation of HV obtained by different base operators with and without the proposed DGCC framework.}
	\label{tab:hv_ablation_pairwise}
	\renewcommand{\arraystretch}{1.2} 
	\resizebox{\textwidth}{!}{
		\begin{tabular}{lcccccc}
			\toprule
			\multirow{2}{*}{Cases} & \multicolumn{2}{c}{NSGA-II Operators} & \multicolumn{2}{c}{HI\_MOEA Operators} & \multicolumn{2}{c}{MMTSP\_DS Operators} \\
			\cmidrule(lr){2-3} \cmidrule(lr){4-5} \cmidrule(lr){6-7} 
			& with DGCC & without DGCC & with DGCC & without DGCC & with DGCC & without DGCC \\
			\midrule
			1  & \textbf{3.740e+08 $\pm$ 7.410e+06} & 1.548e+08 $\pm$ 1.430e+07 (++) & \textbf{3.733e+08 $\pm$ 4.586e+06} & 1.608e+08 $\pm$ 1.864e+07 (++) & \textbf{3.656e+08 $\pm$ 3.770e+06} & 1.224e+08 $\pm$ 1.668e+07 (++) \\
			2  & \textbf{9.781e+08 $\pm$ 1.955e+07} & 2.114e+08 $\pm$ 2.273e+07 (++) & \textbf{9.664e+08 $\pm$ 1.584e+07} & 2.185e+08 $\pm$ 2.681e+07 (++) & \textbf{9.584e+08 $\pm$ 1.188e+07} & 1.538e+08 $\pm$ 2.838e+07 (++) \\
			3  & \textbf{4.163e+09 $\pm$ 4.633e+07} & 1.688e+09 $\pm$ 2.384e+08 (++) & \textbf{4.146e+09 $\pm$ 5.775e+07} & 1.689e+09 $\pm$ 2.323e+08 (++) & \textbf{4.173e+09 $\pm$ 6.463e+07} & 1.408e+09 $\pm$ 3.046e+08 (++) \\
			4  & \textbf{1.338e+10 $\pm$ 1.761e+08} & 2.920e+09 $\pm$ 3.062e+08 (++) & \textbf{1.341e+10 $\pm$ 1.496e+08} & 3.117e+09 $\pm$ 5.452e+08 (++) & \textbf{1.336e+10 $\pm$ 1.619e+08} & 2.180e+09 $\pm$ 7.039e+08 (++) \\
			5  & \textbf{1.059e+14 $\pm$ 1.514e+11} & 1.674e+09 $\pm$ 2.631e+08 (++) & \textbf{1.058e+14 $\pm$ 1.166e+11} & 1.577e+09 $\pm$ 2.960e+08 (++) & \textbf{1.058e+14 $\pm$ 1.501e+11} & 1.587e+09 $\pm$ 2.776e+08 (++) \\
			6  & \textbf{1.061e+14 $\pm$ 1.283e+11} & 2.351e+09 $\pm$ 4.353e+08 (++) & \textbf{1.060e+14 $\pm$ 8.660e+10} & 2.112e+09 $\pm$ 2.754e+08 (++) & \textbf{1.060e+14 $\pm$ 1.469e+11} & 2.170e+09 $\pm$ 3.147e+08 (++) \\
			7  & \textbf{1.089e+14 $\pm$ 2.358e+11} & 3.538e+09 $\pm$ 7.454e+08 (++) & \textbf{1.089e+14 $\pm$ 2.277e+11} & 3.358e+09 $\pm$ 6.055e+08 (++) & \textbf{1.089e+14 $\pm$ 2.362e+11} & 3.602e+09 $\pm$ 6.253e+08 (++) \\
			8  & \textbf{1.090e+14 $\pm$ 2.457e+11} & 4.167e+09 $\pm$ 8.134e+08 (++) & \textbf{1.088e+14 $\pm$ 2.363e+11} & 3.976e+09 $\pm$ 5.214e+08 (++) & \textbf{1.089e+14 $\pm$ 2.058e+11} & 3.799e+09 $\pm$ 5.317e+08 (++) \\
			\bottomrule
		\end{tabular}}
\end{table*}

\subsection{Parameter Analysis}

In this section, we focus on two parameters in the DGCC algorithm, i.e., the component structure adjustment period \( L \) and the available computational resources \( Q \) per iteration. These two parameters respectively influence the key parts of component structure dynamic adjustment and computational resource allocation.

\subsubsection{Structure Adjustment Period}

The structure adjustment period \( L \) determines how frequently the component structure is readjusted, occurring every \( L \) iterations. Given a fixed number of fitness evaluations, each iteration consists of approximately 20 generations, and the algorithm runs for about 30 iterations in total. To analyze the impact of \( L \), we evaluate the normalized HV across different values of \( L = \{1,2, \dots, 30\} \) on three test cases from China, France, and Germany.

As shown in Figure \ref{fig:L_HV}, the HV curves for the three test cases generally exhibit an initial upward trend followed by a decline, with most peaks occurring in the range of \( L = 6 \) to \( 11 \). Since the number of structural adjustments in the algorithm is inversely proportional to \( L \), the results indicate that the algorithm performs best when the adjustment frequency is moderate. When the adjustment frequency is too high (e.g., \( L < 6 \)), frequent structural changes disrupt the optimization process, as the corresponding subpopulations must be restarted after each adjustment, leading to a loss of optimization information and degraded performance. Conversely, when the adjustment frequency is too low (e.g., \( L > 11 \)), optimization effectiveness also deteriorates significantly. This suggests that infrequent adjustments waste computational resources on suboptimal structures, resulting in inferior performance. Notably, when \( L > 15 \), the component structure is adjusted only once throughout the optimization process, and the HV curves for all three test cases exhibit a declining trend, indicating that earlier structural adjustments are beneficial. Although fluctuations appear in the HV curve for the Germany test case when \( L > 15 \), the overall conclusion remains consistent: appropriately balancing the frequency of component structure adjustments is crucial for achieving optimal performance.

\begin{figure}[ht]
	\centering
	\includegraphics[width=0.7\textwidth]{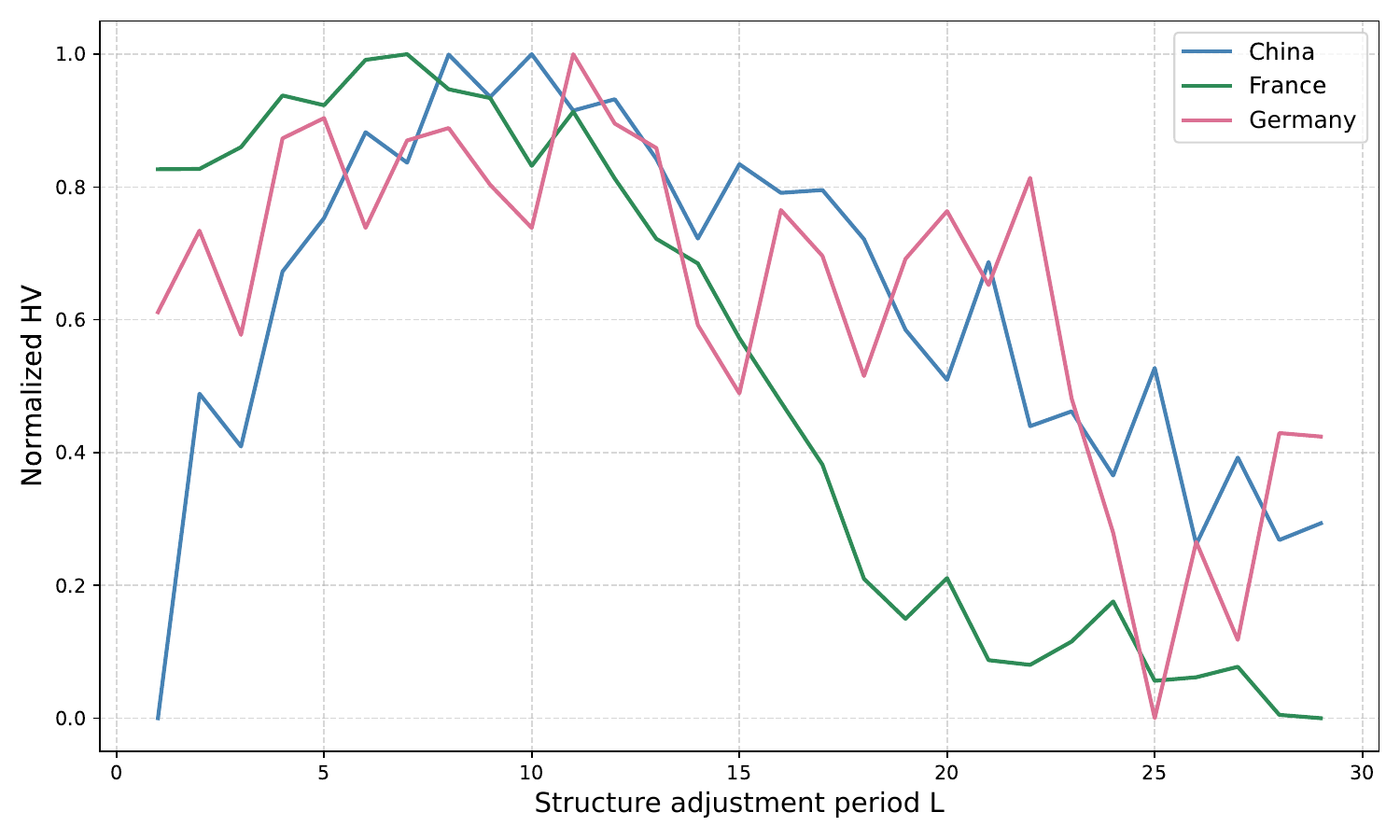}
	\caption{The variation curve of normalized HV with respect to the structure adjustment period \(L\) in three test cases from China, France, and Germany.}
	\label{fig:L_HV}
	\Description{Normalized hypervolume versus structure adjustment period L for test cases from China, France, and Germany. The curves generally rise at smaller periods and decline at larger periods, with fluctuations.}
\end{figure}

\subsubsection{Available Computational Resources}

The available computational resource $Q = (\mathcal{I}_{\text{bas}} + \mathcal{I}_{\text{add}})/n$ represents the average computational budget allocated to each component in one iteration. A larger \( Q \) value increases the computational resources allocated per iteration but reduces the total number of iterations, thereby decreasing the frequency of information exchange between components. To analyze the impact of \( Q \), we evaluate the normalized HV across different values of \( Q = \{10, 20, \dots, 200\} \) on the same three test cases from China, France, and Germany.

As shown in Figure \ref{fig:res_HV}, the HV metric decreases significantly as \( Q \) increases. This decline occurs because a larger \( Q \) reduces the total number of iterations, thereby limiting the frequency of information exchange between components. Moreover, the experimental results support the necessity of frequently switching component optimization. This is primarily due to the adoption of component-wise isolated evaluation, where each component is evaluated independently with high accuracy, minimizing the impact of interactions. Unlike global evaluation, which requires longer generations to mitigate inaccuracies caused by interactions, our approach benefits from frequent evaluations without such interference. Overall, the results demonstrate that in DGCC, increasing the frequency of component information exchange enhances the algorithm's performance.

\begin{figure}[ht]
	\centering
	\includegraphics[width=0.7\textwidth]{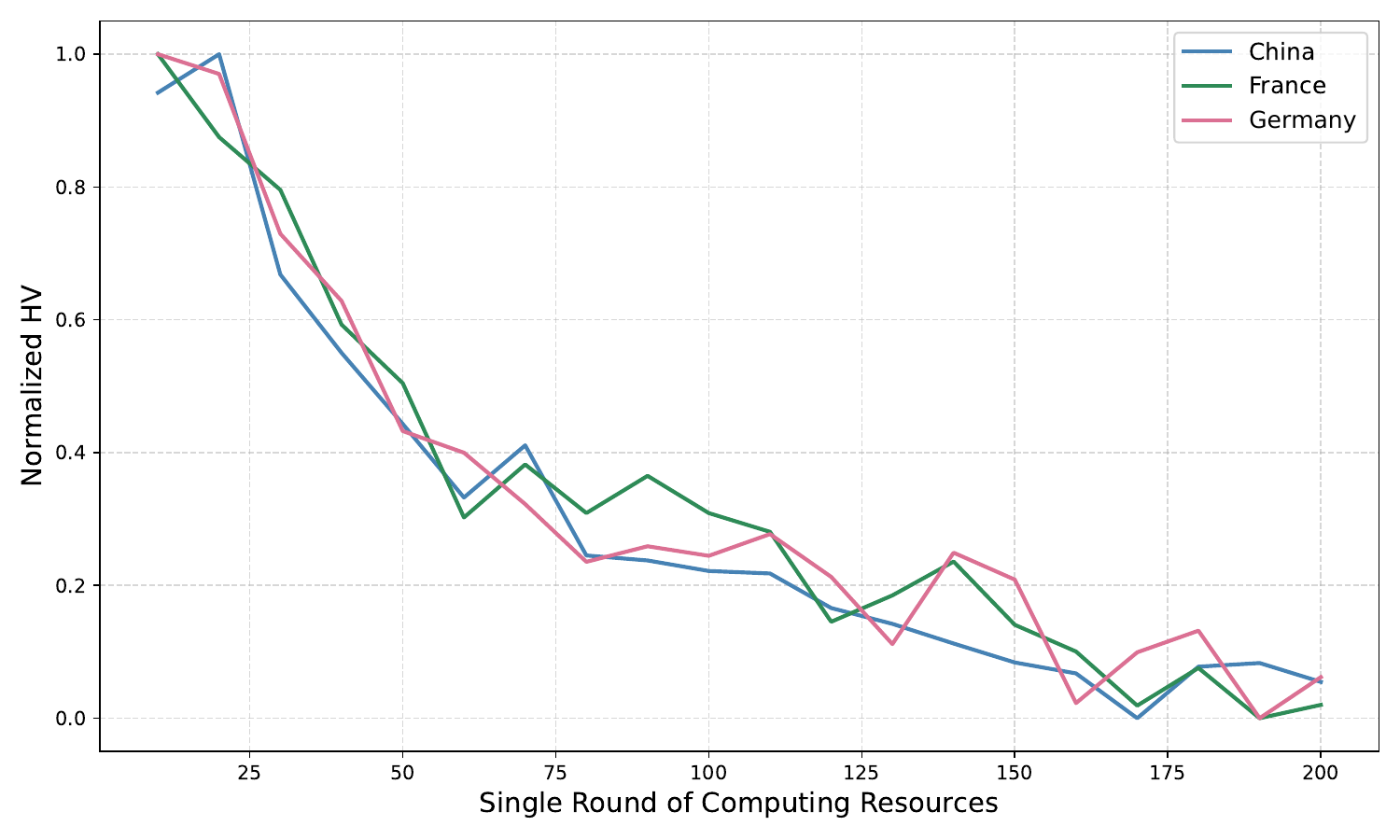}
	\caption{The variation curve of normalized HV with respect to the available computational resource \( Q \) in three test cases from China, France, and Germany.}
	\label{fig:res_HV}
	\Description{Normalized hypervolume versus computational resources per round Q for test cases from China, France, and Germany. All three curves show an overall decrease as Q increases, with local fluctuations.}
\end{figure}

\subsection{Qualitative Analysis}

\begin{table*}[ht]
	\centering
	\caption{Itineraries for DGCC, CCIP, and LSNSGA-II}
	\label{combined_itinerary}
	\resizebox{0.97\textwidth}{!}{\begin{tabular}{c c c p{12cm}}
			\toprule
			Algorithms & Day & City  & Points of Interest \\ \midrule
			\multirow{4}{*}{\raisebox{-9ex}{DGCC}} & \multirow{2}{*}{1} & \multirow{2}{*}{Nanjing} & Laomen East Historical District $\to$ Confucius Temple--Qinhuai Scenic Area $\to$ Confucius Temple $\to$ Zhan Garden $\to$ Mochou Lake Scenic Area \\
			\cmidrule(lr){2-4}
			& \multirow{2}{*}{2} & \multirow{2}{*}{Nanjing} & Shiziqiao Food Street $\to$ Jiming Temple $\to$ Jiangning Imperial Silk Manufacturing Museum $\to$ Xinjiekou \\
			\cmidrule(lr){2-4}
			& 3 & Nanjing & Presidential Palace of Nanjing $\to$ Nanjing Ming Palace Ruins Park \\
			\cmidrule(lr){2-4}
			& 4 & Suzhou & Fengqiao Scenic Area $\to$ Hanshan Temple $\to$ Suzhou Museum $\to$ Lion Grove Garden \\
			\midrule
			
			\multirow{4}{*}{\raisebox{-5ex}{CCIP}}  & 1 & Nanjing & Drum Tower $\to$ Jiming Temple $\to$ Xinjiekou $\to$ Mochou Lake Scenic Area \\
			\cmidrule(lr){2-4}
			& \multirow{2}{*}{2} & \multirow{2}{*}{Nanjing} & Qinhuai River $\to$ Laomen East Historical District $\to$ Nanjing Ming Palace Ruins Park $\to$ Nanjing Museum \\
			\cmidrule(lr){2-4}
			& 3 & Suzhou    & Lingering Garden $\to$ Hanshan Temple $\to$ Guanqian Street \\
			\cmidrule(lr){2-4}
			& 4 & Suzhou    & Suzhou Museum $\to$ Couple's Retreat Garden $\to$ Pingjiang Historical Street \\
			\midrule
			
			\multirow{4}{*}[-7ex]{LSNSGA-II} & \multirow{3}{*}{1} & \multirow{3}{*}{Nanjing} & Dacheng Hall of Confucius Temple $\to$ Confucius Temple–Qinhuai Scenic Area $\to$ China Imperial Examination Museum (Jiangnan Gongyuan) $\to$ Qinhuai River Painting Boat Ride $\to$ Wuyi Lane \\
			\cmidrule(lr){2-4}
			& \multirow{2}{*}{2} & \multirow{2}{*}{Suzhou}    & Couple's Retreat Garden $\to$ Eslite Bookstore $\to$ Guanqian Street $\to$ Jinji Lake $\to$ Suzhou Amusement Park Forest World \\
			\cmidrule(lr){2-4}
			& 3 & Nanjing & Mochou Lake Scenic Area $\to$ Confucius Temple $\to$ Zhan Garden \\
			\cmidrule(lr){2-4}
			& \multirow{2}{*}{4} & \multirow{2}{*}{Nanjing} & Chaotian Palace $\to$ Zhonghua Gate Castle $\to$ Ruins of the Great Bao’e Temple Scenic Area $\to$ Nanjing City Wall $\to$ Yuyuan Garden \\
			\bottomrule
	\end{tabular}}
\end{table*}

In this set of experiments, we conducted a qualitative analysis of real-world cases for the DGCC algorithm and the two best-performing algorithms, CCIP and LSNSGA-II. We used these three algorithms to plan a four-day itinerary to Nanjing and Suzhou, and selected representative cases for analysis. The itineraries generated by the three algorithms are summarized in Table~\ref{combined_itinerary}. Since the multi-objective optimization yields a set of non-dominated solutions, selecting a single itinerary cannot fully represent the absolute search capability of the underlying optimizers. Instead, this qualitative analysis aims to demonstrate how different algorithmic frameworks allocate temporal and spatial resources, and to explain the fundamental reasons behind their performance disparities in practical scenarios.

The results indicate that LSNSGA-II, which is designed for small-scale itinerary planning, produced an itinerary where the trip starts in Nanjing, moves to Suzhou, and then returns to Nanjing. This is evidently impractical, as it introduces unnecessary intercity transitions. The issue arises because LSNSGA-II lacks the ability to distinguish between different cities in multi-city planning, leading to inefficient travel routes. Both DGCC and CCIP are algorithms designed for large-scale itinerary planning. In this representative solution, DGCC allocated three days to Nanjing and one day to Suzhou, whereas CCIP allocated two days to each city. This difference reflects the dynamic structure adjustment mechanism of DGCC, which allows the allocation of travel days to adapt to the evaluated contributions of different city components. Overall, DGCC demonstrated greater advantages in real-world cases compared to CCIP and LSNSGA-II.
\section{Discussion}
\label{sec:Discussion}
The experimental results support the effectiveness of DGCC at the framework level. The ablation study (Table~\ref{tab:ablation_hv}) shows that all three mechanisms, namely dynamic component structure adjustment, computational resource allocation, and population inheritance, contribute to the overall performance, with the complete DGCC achieving higher mean HV than each ablated variant on all 18 itinerary planning test cases. Furthermore, in the optimizer-decoupling experiments (Table~\ref{tab:hv_ablation_pairwise}), combining DGCC with each of the three base optimizers, NSGA-II, HI-MOEA, and MMTSP-DS, yields higher mean HV than the corresponding standalone optimizer across all eight evaluated configurations. These results indicate that the performance advantage of DGCC does not arise solely from a particular underlying optimizer, but is also associated with its exploitation of component structures and its coordination of the search processes across different components.\\
In practical scenarios, multiple airports or other candidate connection points make the role of component boundaries more explicit. Under the static edge-cost model adopted in this work, when the departure and arrival nodes between two components are fixed, the corresponding inter-component connection cost remains unchanged with respect to adjustments of the internal POI sequence, and component-level search can therefore still be conducted under the given boundary connections. In contrast, if boundary nodes such as airports, entrances, or exits are themselves decision variables, their selection affects both the inter-component connection cost and the intra-component route. In this case, more direct coupling arises between the boundary nodes and the internal component subpaths, which therefore need to be evaluated jointly during optimization. Thus, multiple candidate connection points do not in themselves fall outside the representational scope of the current graph model, but when boundary nodes participate in the decision process, the existing encoding and component coordination mechanisms need to be extended accordingly.\\
The granularity of component partitioning should likewise match the actual interaction structure of the problem. Research on large-scale cooperative coevolution has shown that variable interactions and grouping strategies are important foundations for problem decomposition, while overlapping components or shared interacting variables further increase the difficulty of decomposition \cite{yang2008large,sun2019decomposition}. In the problem studied here, a city is only one natural instance of a component in the current scenario and is not the only theoretically possible basic unit. When a large city contains finer-grained spatial or transportation structures, it may be further divided into smaller components, but the new partition must still preserve the original visitation requirements and time-allocation constraints, and the corresponding graph assumptions and weak decomposability conditions need to be re-examined. For problems that only approximately exhibit such structures, the current theory cannot provide the same guarantees as when the weak decomposability conditions are fully satisfied, although the existing component partition and the decomposition scheme of DGCC may still serve as a structural prior.\\
Beyond itinerary planning, structural characteristics similar to those exploited in this work may also arise in other problems. For example, studies on the clustered orienteering problem have considered application scenarios such as distribution and area-based waste collection \cite{angelelli2014clustered}. In these problems, service locations, customer groups, and travel connections between different groups can respectively correspond to the POIs, components, and inter-component connections considered in this work. For such problems, the decomposition and cooperative search ideas of DGCC have relatively direct modeling correspondences, but it is still necessary to re-examine whether the optimal routes are compatible with the adopted contiguous component-visit structure under the specific service rules, reward definitions, and constraint formulations. On the other hand, in the distributed resource-constrained multi-project scheduling problem, different projects have relatively independent internal activities and scheduling structures, while still requiring coordination through shared global resources \cite{li2018multi}. This provides a more general algorithmic analogy to DGCC's mechanism for coordinating local component optimization at the global level. These examples suggest that the ideas of component decomposition, local optimization, and global cooperation adopted in this work may have the potential to be extended to other large-scale optimization problems. However, such structural correspondences do not imply that these problems already satisfy the weak decomposability conditions defined in this work for path-based problems, and transferring the corresponding theoretical results still requires the relevant assumptions to be re-established and verified for the target problem.

\section{Conclusion}
\label{sec: Conclusion}
This paper presents DGCC, a decomposability-guided cooperative coevolutionary framework for large-scale multi-objective itinerary planning. We first analyze the decomposable structure of large-scale itinerary planning and, based on this analysis, design dynamic component structure adjustment, computational resource allocation, and population inheritance mechanisms to coordinate the search processes of interacting components. Across 18 real-world itinerary planning test cases, eight large-scale semi-synthetic test cases, and eight test cases adapted from public clustered orienteering problem benchmarks, DGCC achieves the highest mean HV among the five evaluated algorithms on all 34 test cases. The ablation study further validates the contributions of the core mechanisms to the overall performance, while the optimizer-decoupling experiments based on NSGA-II, HI-MOEA, and MMTSP-DS show that combining the DGCC framework with different underlying optimizers yields consistent performance improvements across the evaluated configurations. These results indicate that exploiting component structures and coordinating component-level search at the global level provide an effective approach for addressing the large-scale itinerary planning problems investigated in this work.

In future work, we will further investigate adaptive component partitioning and more flexible boundary-node coordination mechanisms for general scenarios involving multiple candidate connection points or less clearly defined component boundaries, such as enhancing the search around component boundaries and allowing limited cross-component adjustments. We will also explore the potential extension of the DGCC framework to other large-scale optimization problems with similar local-global coupling structures.

\begin{acks}
	{\emergencystretch=3em
This work is partly supported by the National Natural Science Foundation of China (Grant No. 62602284, U23B2058, 62576121), Shenzhen Fundamental Research Program (Grant No. JCYJ20220818102414030), Shenzhen Science and Technology Program (Grant No. ZDSYS20210623091809029).\par}
\end{acks}
\bibliographystyle{ACM-Reference-Format}
\bibliography{references}

\end{document}